\documentclass[10pt,twocolumn,letterpaper]{article}
\usepackage{caption}
\usepackage{subcaption}

\usepackage{cvpr}
\usepackage{times}
\usepackage{epsfig}
\usepackage{graphicx}
\usepackage{amsmath}
\usepackage{amssymb}
\usepackage{amsthm}
\usepackage[normalem]{ulem} %
\usepackage{xcolor}
\usepackage{multirow}
\usepackage{rotating}
\usepackage{xr}

\usepackage{microtype}

\DeclareMathOperator{\rank}{rank}

\usepackage[bookmarks=false]{hyperref}

\newcommand{\myparagraph}[1]{\vspace{3pt}\noindent{\bf #1}}

\definecolor{aqua}{rgb}{0.0, 0.48, 0.65}

\newcommand{\ournet}{AtlasNet}
\newcommand{\patches}{learnable parametrizations}

 \cvprfinalcopy %

\def\cvprPaperID{1775} %

\ifcvprfinal\fi
\begin{document}

\title{AtlasNet: A Papier-M\^ach\'e Approach to Learning 3D Surface Generation}

\author{
Thibault Groueix$^{1\thanks{Work done at Adobe Research during TG's summer internship}}$, Matthew Fisher$^2$, Vladimir G. Kim$^2$, Bryan C. Russell$^2$, Mathieu Aubry$^1$\\
$^1$LIGM (UMR 8049), \'Ecole des Ponts, UPE, $^2$Adobe Research\\
{\tt\small \url{http://imagine.enpc.fr/~groueixt/atlasnet/}}\\
}

\maketitle
\begin{abstract}
\vspace{-3mm}

We introduce a method for learning to generate the surface of 3D shapes. 
Our approach represents a 3D shape as a collection of parametric surface elements and, in contrast to methods generating voxel grids or point clouds, naturally infers a surface representation of the shape.
   Beyond its novelty, our new shape generation framework, \ournet{}, comes with significant advantages, such as improved precision and generalization capabilities, and the possibility to generate a shape of arbitrary resolution without memory issues. 
We demonstrate these benefits and compare to strong baselines on the ShapeNet benchmark for two applications: (i) auto-encoding shapes, and (ii) single-view reconstruction from a still image.
   We also provide results showing its potential for other applications, such as morphing, parametrization, super-resolution, matching, and co-segmentation.
\end{abstract}
\vspace{-6mm}

\section{Introduction}
\vspace{-2mm}
Significant progress has been made on learning good representations for images, allowing impressive applications in image generation~\cite{Isola:2017,Zhu:2017:cylcegan}. However, learning a representation for generating high-resolution 3D shapes remains an open challenge. Representing a shape as a volumetric function~\cite{choy20163d,Hane:2017,TDB17b} only provides voxel-scale sampling of the underlying smooth and continuous surface. In contrast, a point cloud~\cite{qi2016pointnet,Qi:2017:nips} provides a representation for generating on-surface details~\cite{Fan:2017:cvpr}, efficiently leveraging sparsity of the data. However, points do not directly represent neighborhood information, making it difficult to approximate the smooth low-dimensional manifold structure with high fidelity. 

To remedy shortcomings of these representations, surfaces are a popular choice in geometric modeling. A surface is commonly modeled by a  polygonal mesh: a set of vertices, and a list of triangular or quad primitives composed of these vertices, providing piecewise planar approximation to the smooth manifold. Each mesh vertex contains a 3D (XYZ) coordinate, and, frequently, a 2D (UV) embedding to a plane. The UV parameterization of the surface provides an effective way to store and sample functions on surfaces, such as normals, additional geometric details, textures, and other reflective properties such as BRDF and ambient occlusion. 
One can imagine converting point clouds or volumetric functions produced with existing learned generative models as a simple post-process. %
However, this requires solving two fundamental, difficult, and long-standing challenges in geometry processing: global surface parameterization and meshing.

\begin{figure*}[t!]
\centering
\begin{subfigure}[b]{0.32\linewidth}
 \includegraphics[width=\linewidth]{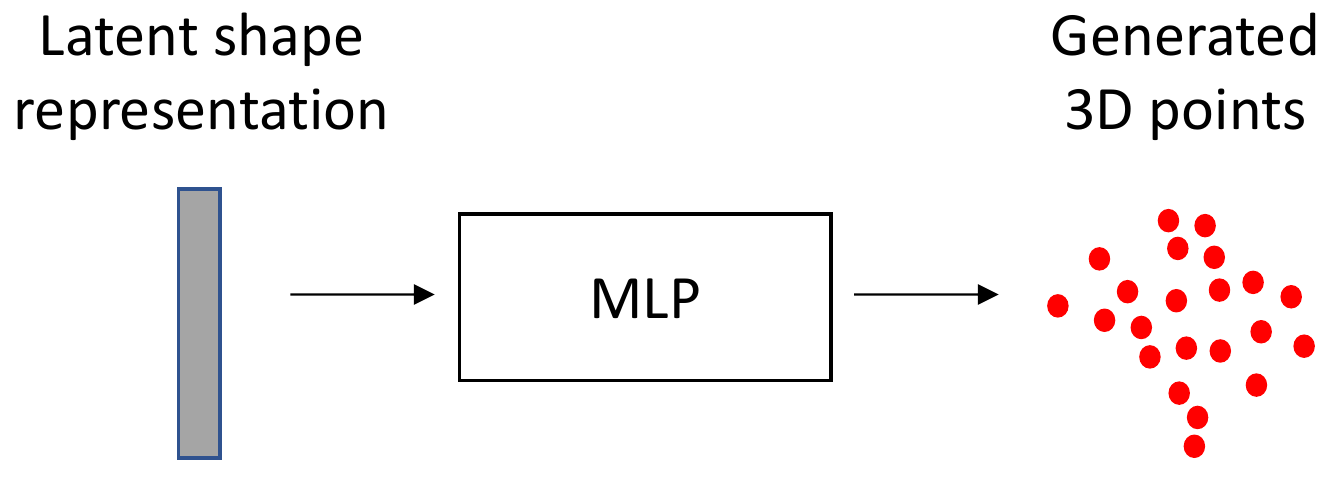}
\caption{Points baseline. \label{fig:baseline}}
\end{subfigure}
~\vline~
\begin{subfigure}[b]{0.32\linewidth}
 \includegraphics[width=\linewidth]{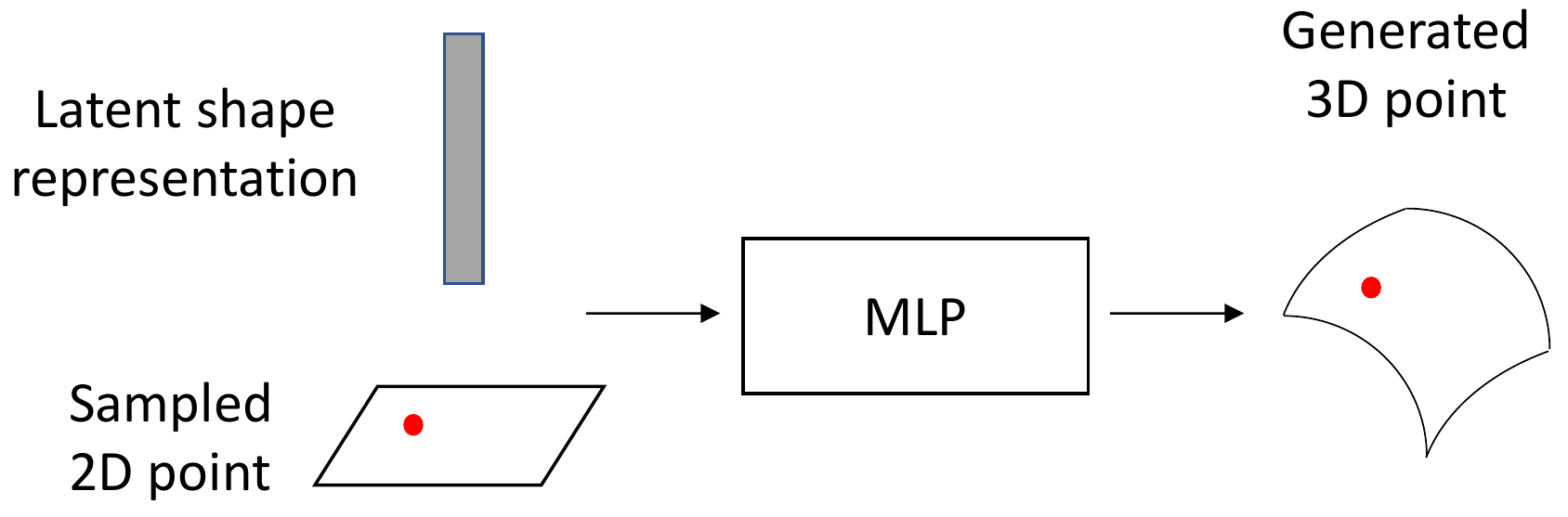}
\caption{Our approach with one patch. \label{fig:ours1}}
\end{subfigure}
~\vline~
\begin{subfigure}[b]{0.32\linewidth}
 \includegraphics[width=\linewidth]{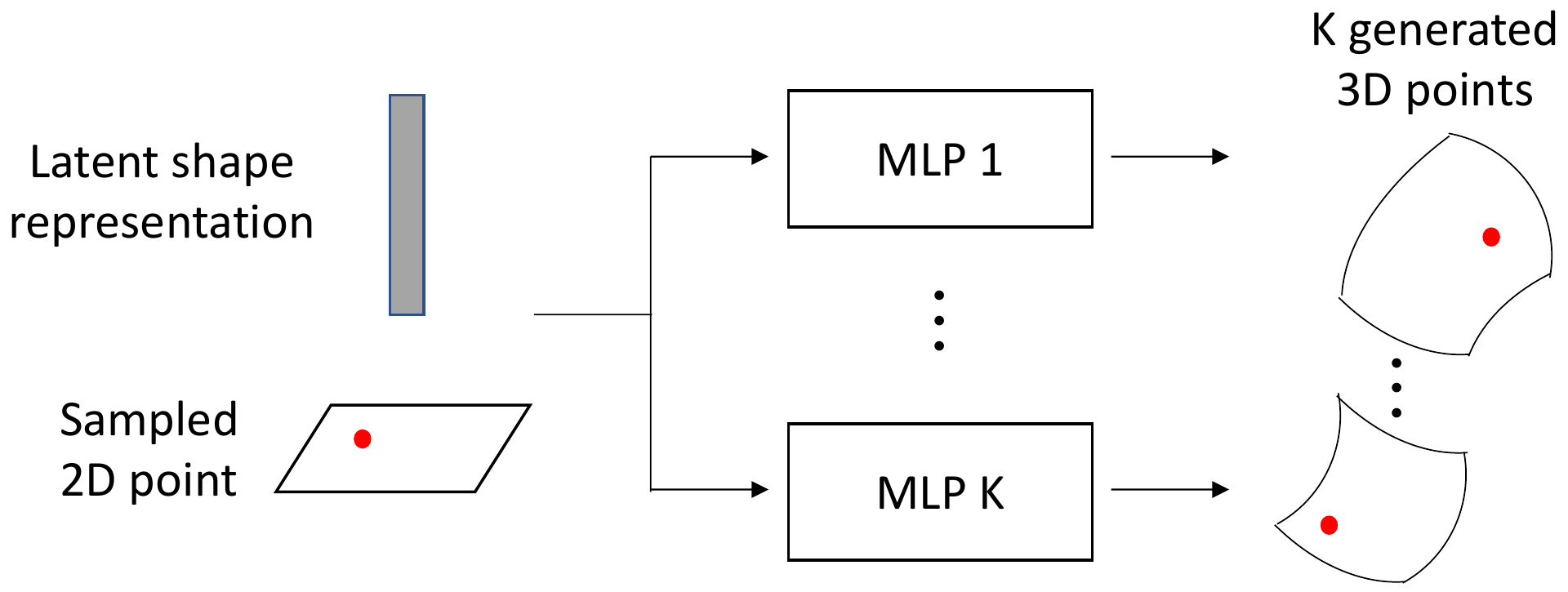}
\caption{Our approach with K patches. \label{fig:oursK}}
\end{subfigure}
\caption{{\bf Shape generation approaches.} All methods take as input a latent shape representation (that can be learned jointly with a reconstruction objective) and generate as output a set of points. (a) A baseline deep architecture would simply decode this latent representation into a set of points of a given size. (b) Our approach takes as additional input a 2D point sampled uniformly in the unit square and uses it to generate a single point on the surface. Our output is thus the continuous image of a planar surface. In particular, we can easily infer a mesh of arbitrary resolution on the generated surface elements. (c) This strategy can be repeated multiple times to represent a 3D shape as the union of several surface elements.
}
  \label{fig:overview}
  \vspace*{-4mm}
\end{figure*}

In this paper we explore learning the surface representation directly.  Inspired by the formal definition of a surface as a topological space that locally resembles the Euclidean plane, we seek to approximate the target surface locally by mapping a set of squares to the surface of the 3D shape. The use of multiple such squares allows us to model complex surfaces with non-disk topology. Our representation of a shape is thus extremely similar to an atlas, as we will discuss in Section \ref{sec:theory}. The key strength of our method is that it jointly learns a parameterization and an embedding of a shape. This helps in two directions. First, by ensuring that our 3D points come from 2D squares we favor learning a continuous and smooth 2-manifold structure. Second, by generating a UV parameterization for each 3D point, we generate a global surface parameterization, which is key to many applications such as texture mapping and surface meshing. Indeed, to generate the mesh, we simply transfer a regular mesh from our 2D squares to the 3D surface, and to generate a regular texture atlas, we simply optimize the metric of the square to become as-isometric-as-possible to the corresponding 3D shape (Fig.~\ref{fig:teaser_fig}). 

Since our work deforms primitive surface elements into a 3D shape, it can be seen as bridging the gap between the recent works that learn to represent 3D shapes as a set of simple primitives, with a fixed, low number of parameters \cite{tulsiani2016learning} and those that represent 3D shapes as an unstructured set of points~\cite{Fan:2017:cvpr}. It can also be interpreted as learning a factored representation of a surface, where a point on the shape is represented jointly by a vector encoding the shape structure and a vector encoding its position. Finally, it can be seen as an attempt to bring to 3D the power of convolutional approaches for generating 2D images~\cite{Isola:2017,Zhu:2017:cylcegan} by sharing the network parameters for parts of the surface.

\myparagraph{Our contributions.} 
In this paper:
\begin{itemize}
  \vspace*{-1mm}
    \item We propose a novel approach to 3D surface generation, dubbed {\it \ournet}, which is composed of a union of \patches. These \patches~transform a set of 2D squares to the surface, covering it in a way similar to placing strips of paper on a shape to form a papier-m\^ach\'e. The parameters of the transformations come both from the learned weights of a neural network and a learned representation of the shape. 
      \vspace*{-1mm}
    \item We show that the learned parametric transformation maps locally everywhere to a surface, naturally adapts to its underlying complexity, can be sampled at any desired resolution, and allows for the transfer of a tessellation or texture map to the generated surface. 
      \vspace*{-1mm}
    \item We demonstrate the advantages of our approach both qualitatively and quantitatively on high resolution surface generation from (potentially low resolution) point clouds and 2D images
      \vspace*{-1mm}
    \item We demonstrate the potential of our method for several applications, including shape interpolation, parameterization, and shape collections alignment. 
\end{itemize}
All the code is available at the project webpage\footnote{\url{https://github.com/ThibaultGROUEIX/AtlasNet}.}. 

\section{Related work}
\label{sec:related}

3D shape analysis and generation has a long history in computer vision. In this section, we only discuss the most directly related works for representation learning for 2-manifolds and 3D shape generation using deep networks.
\myparagraph{Learning representations for 2-manifolds.}
A polygon mesh is a widely-used representation for the 2-manifold surface of 3D shapes.
Establishing a connection between the surface of the 3D shape and a 2D domain, or surface parameterization, 
is a long-standing problem in geometry processing, 
with applications in texture mapping, re-meshing, and shape correspondence~\cite{Hormann08}.
Various related representations have been used for applying neural networks on surfaces. 
The geometry image representation~\cite{Gu02,Sander03} views 3D shapes as functions (e.g., vertex positions) 
embedded in a 2D domain, providing a natural input for 2D neural networks~\cite{Sinha2016}. 
Various other parameterization techniques, such as local polar 
coordinates~\cite{Masci15,Boscaini16} and global seamless maps~\cite{Maron17}
have been used for deep learning on 2-manifolds. 
Unlike these methods, we do not need our input data to be parameterized. Instead, we learn 
the parameterization directly from point clouds. Moreover, these methods assume that the training and testing data are 2-manifold meshes, and thus cannot easily be used for surface reconstructions from point clouds or images.

\myparagraph{Deep 3D shape generation.} 
Non-parametric approaches retrieve shapes from a large corpus~\cite{Bansal:2016,Li:2015:purification,Massa:2016}, but require having an exact instance in the corpus. 
One of the most popular shape representation for generation is the voxel representation. %
Methods for generating a voxel grid have been demonstrated with various inputs, namely one or several images~\cite{choy20163d,Girdhar16b}, full 3D objects in the form of voxel grids \cite{Girdhar16b,wu20153d}, and 3D objects with missing shape parts~\cite{wu20153d,HRSC:2017}. 
Such direct volumetric representation is costly in term of memory and is typically limited to coarser resolutions.
To overcome this, recent work has looked at a voxel representation of the surface of a shape via oct-trees~\cite{Hane:2017,Riegler2017THREEDV,TDB17b}. Recently, Li et al. also attempted to address this issue via learning to reason over hierarchical procedural shape structures and only generating voxel representations at the part level~\cite{li_sig17}.
As an alternative to volumetric representations, another line of work has learned to encode~\cite{qi2016pointnet,Qi:2017:nips} and decode~\cite{Fan:2017:cvpr} a 3D point representation of the surface of a shape. 
A limitation of the learned 3D point representation is there is no surface connectivity (e.g., triangular surface tessellation) embedded into the representation. %

Recently, Sinha et al.~\cite{Sinha2017} proposed to use a spherical parameterization of a single deformable mesh (if available) or of a few base shapes (composed with authalic projection of a sphere to a plane) to represent training shapes as parameterized meshes. They map vertex coordinates to the resulting UV space and use 2D neural networks for surface generation. This approach relies on consistent mapping to the UV space, and thus requires automatically estimating correspondences from training shapes to the base meshes (which gets increasingly hard for heterogeneous datasets). Surfaces generated with this method are also limited to the topology and tessellation of the base mesh. Overall, learning to generate surfaces of arbitrary topology from unstructured and heterogeneous input still poses a challenge.

\section{Locally parameterized surface generation}

In this section, we detail the theoretical motivation for our approach and present some theoretical guarantees.

\label{sec:theory}
We seek to learn to generate a surface of a 3D shape. 
A subset $\mathcal{S}$ of $\mathbb{R}^3$ is a \textit{2-manifold} if, for every point $\mathbf{p}\in\mathcal{S}$, there is an open set $U$ in $\mathbb{R}^2$ and an open set $W$ in $\mathbb{R}^3$ containing $\mathbf{p}$ such that $\mathcal{S}\cap W$ is homeomorphic to $U$. 
The set homeomorphism from $\mathcal{S}\cap W$ to $U$ is called a {\it chart}, and its inverse a {\it parameterization}.
A set of charts such that their images cover the 2-manifold is called an {\it atlas} of the 2-manifold. 
The ability to learn an atlas for a 2-manifold would allow a number of applications, such as transfer of a  tessellation to the 2-manifold for meshing and texture mapping (via texture atlases). In this paper, we use the word \textit{surface} in a slightly more generic sense than \textit{2-manifold}, allowing for self-intersections and disjoint sets.

We consider a local parameterization of a 2-manifold and explain how we learn to approximate it. More precisely, let us consider a 2-manifold $\mathcal{S}$, a point  $\mathbf{p}\in\mathcal{S}$ and a parameterization $\varphi$ of $\mathcal{S}$ in a local neighborhood of $\mathbf{p}$. We can assume that $\varphi$ is defined on the open unit square $]0,1[^2$ by first restricting $\varphi$ to an open neighborhood of $\varphi^{-1}(\mathbf{p})$ with disk topology where it is defined (which is possible because $\varphi$ is continuous) and then mapping this neighborhood to the unit square. %

We pose the problem of learning to generate the local 2-manifold previously defined as one of finding a parameterizations $\varphi_{\theta}(x)$ with parameters $\theta$ which map the open unit 2D square $]0,1[^2$ to a good approximation of the desired 2-manifold $\mathcal{S}_{\text{loc}}$. 
Specifically, calling $\mathcal{S}_{\theta}= \varphi_{\theta}(]0,1[^2)$,
we seek to find parameters $\theta$ minimizing the following objective function,
\begin{equation}
    \min_\theta \mathcal{L}\left(\mathcal{S}_{\theta}, \mathcal{S}_{\text{loc}}\right) + \lambda\mathcal{R}\left(\theta\right),
\end{equation}
where $\mathcal{L}$ is a loss over 2-manifolds, $\mathcal{R}$ is a regularization function over parameters $\theta$, and $\lambda$ is a scalar weight. 
In practice, instead of optimizing a loss over 2-manifolds $\mathcal{L}$, we optimize a loss over point sets sampled from these 2-manifolds such as Chamfer and Earth-Mover distance. 

One question is, how do we represent the functions $\varphi_{\theta}$? A good family of functions should (i) generate 2-manifolds and (ii) be able to produce a good approximation of the desired 2-manifolds $S_{\text{loc}}$. 
We show that multilayer perceptrons (MLPs) with rectified linear unit (ReLU) nonlinearities almost verify these properties, and thus are an adequate family of functions.
Since it is difficult to design a family of functions that always generate a 2-manifold, we relax this constraint and consider functions that locally generate a 2-manifold.

\newtheorem{proposition}{Proposition}

\begin{proposition}
Let $f$ be a multilayer perceptron with ReLU nonlinearities. There exists a finite set of polygons $P_i$, $i\in \lbrace 1,\ ...,N\rbrace$ such that on each $P_i$ $f$ is an affine function: $\forall x\in P_i, \ f(x)=A_i x+b$, where $A_i$ are $3\times2$ matrices. If for all $i$, $\rank(A_i)=2$, then for any point  $\mathbf{p}$ in the interior of one of the $P_is$ there exists a neighborhood $\mathcal{N}$ of $\mathbf{p}$ such that $f(\mathcal{N})$ is a 2-manifold. 
\end{proposition}

\begin{proof}
The fact that $f$ is locally affine is a direct consequence of the fact that we use ReLU non-linearities. If $\rank(A_i)=2$ the inverse of $A_i x+b$ is well defined on the surface and continuous, thus the image of the interior of each $P_i$ is a 2-manifold.  
\end{proof}
To draw analogy to texture atlases in computer graphics, we call the local functions we learn to approximate a 2-manifold {\it learnable parameterizations} and the set of these functions $A$ a {\it learnable atlas}. 
Note that in general, an MLP locally defines a rank 2 affine transformation and thus locally generates a 2-manifold, but may not globally as it may 
intersect or overlap with itself. The second reason to choose MLPs as a family is that they can allow us to approximate any continuous surface.

\begin{proposition}
Let $S$ be a 2-manifold that can be parameterized on the unit square. For any $\epsilon>0$ there exists an integer $K$ such that a multilayer perceptron with ReLU non linearities and $K$ hidden units can approximate $S$ with a precision $\epsilon$.
\end{proposition}
\vspace*{-13pt}
\begin{proof}
This is a consequence of the universal representation theorem \cite{hornik1991approximation}
\end{proof}
\vspace*{-8pt}

In the next section, we show how to train such MLPs to align with a desired surface. 

\section{\ournet}
\label{sec:implementation}

In this section we introduce our model, \ournet{}, which decodes a 3D surface given an encoding of a 3D shape. This encoding can come from many different representations such as a point cloud or an image (see Figure~\ref{fig:teaser_fig} for examples).

\newcommand{\groundtruth}{\mathcal{S}^\star}
\newcommand{\patchinput}{\mathcal{A}}
\newcommand{\shapefeature}{\mathbf{x}}
\newcommand{\mlp}{\varphi}
\newcommand{\lossmlp}{\mathcal{L}}
\newcommand{\lossatlas}{\mathcal{L}^\prime}
\newcommand{\parameters}{\theta}
\newcommand{\parameterset}{\Theta}

\subsection{Learning to decode a surface}

Our goal is, given a feature representation $\shapefeature$ for a 3D shape, to generate the surface of the shape. 
As shown in Section~\ref{sec:theory}, an MLP with ReLUs $\mlp_\parameters$ with parameters $\parameters$ can locally generate a surface by learning to map points in $\mathbb{R}^2$ to surface points in $\mathbb{R}^3$. To generate a given surface, we need several of these learnable charts to represent a surface. In practice, we consider $N$ learnable parameterizations $\phi_{\theta_i}$ for $i \in \lbrace 1,\ ...,N\rbrace$.
To train the MLP parameters $\theta_i$, we need to address two questions: (i) how to define the distance between the generated and target surface, and (ii) how to account for the shape feature $\shapefeature$ in the MLP?
To represent the target surface, we use the fact that, independent of the representation that is available to us, we can sample points on it. Let $\patchinput$ be a set of points sampled in the unit square $[0,1]^2$ and $\groundtruth$  a set of points sampled on the target surface.
Next, we incorporate the shape feature $\shapefeature$ by simply concatenating them with the sampled point coordinates $\mathbf{p}\in\patchinput$ before passing them as input to the MLPs. 
Our model is illustrated in Figure~\ref{fig:ours1}. Notice that the MLPs are not explicitly prevented from encoding the same area of space, but their union should cover the full shape. Our MLPs do depend on the random initialization, but similar to convolutional filter weights the network learns to specialize to different regions %
in the output without explicit biases.
We then minimize the Chamfer loss between the set of generated 3D points and $\groundtruth$,
\begin{multline}
    \lossmlp(\parameters) = \sum_{\mathbf{p}\in\patchinput}\sum_{i=1}^N  \min_{\mathbf{q}\in\groundtruth} \left|\phi_{\theta_i}\left(\mathbf{p}; \shapefeature\right) - \mathbf{q}\right|^2\\
    + \sum_{\mathbf{q}\in\groundtruth}  \min_{i \in \lbrace 1,\ ...,N\rbrace}\min_{\mathbf{p}\in\patchinput}\left|\phi_{\theta_i}\left(\mathbf{p}; \shapefeature\right) - \mathbf{q}\right|^2.
\label{eqn:atlas_loss}
\end{multline}

\subsection{Implementation details}

We consider two tasks: (i) to auto-encode a 3D shape given an input 3D point cloud, and (ii) to reconstruct a 3D shape given an input RGB image. For the auto-encoder, we used an encoder based on PointNet~\cite{qi2016pointnet}, which has proven to be state of the art on point cloud analysis on ShapeNet and ModelNet40 benchmarks. 
This encoder transforms an input point cloud into a latent vector of dimension $k=1024$. We experimented with input point clouds of 250 to 2500 points. For images, we used ResNet-18 \cite{he2016deep} as our encoder. The architecture of our decoder is 4 fully-connected layers of size 1024, 512, 256, 128 with ReLU non-linearities on the first three layers and tanh on the final output layer. We always train with output point clouds of size 2500 evenly sampled across all of the learned parameterizations -- scaling above this size is time-consuming because our implementation of Chamfer loss has a compute cost that is quadratic in the number of input points. 
We experimented with different basic weight regularization options but did not notice any generalization improvement.
Sampling of the learned parameterizations as well as the ground truth point-clouds is repeated at each training step to avoid over-fitting. To train for single-view reconstruction, we obtained the best results by training the encoder and using the decoder from the point cloud autoencoder with fixed parameters.
Finally, we noticed that sampling points regularly on a grid on the learned parameterization yields better performance than sampling points randomly. All results used this regular sampling.

\subsection{Mesh generation}
\label{sec:mesh}
The main advantage of our approach is that during inference, we can easily generate a mesh of the shape.

\myparagraph{Propagate the patch-grid edges to the 3D points.} 
The simplest way to generate a mesh of the surface is to transfer a regular mesh on the unit square to 3D, connecting in 3D the images of the points that are connected in 2D. Note that our method allows us to generate such meshes at very high resolution, without facing memory issues, since the points can be processed in batches. We typically use 22500 points. As shown in the results section, such meshes are satisfying, but they can have several drawbacks: they will not be closed, may have small holes between the images of different learned parameterizations, and different patches may overlap. %

\myparagraph{Generate a highly dense point cloud and use Poisson surface reconstruction (PSR) ~\cite{kazhdan2013screened}.}
To avoid the previously mentioned drawbacks, we can additionally densely sample the surface and use a mesh reconstruction algorithm. We start by generating a surface at a high resolution, as explained above. We then shoot rays at the model from infinity and obtain approximately 100000 points, together with their oriented normals, and then can use a standard oriented cloud reconstruction algorithm such as PSR to produce a triangle mesh. We found that high quality normals as well as high density point clouds are critical to the success of PSR, which are naturally obtained using this method. %

\myparagraph{Sample points on a closed surface rather than patches.} To obtain a closed mesh directly from our method, without requiring the PSR step described above, we can sample the input points from the surface of a 3D sphere instead of a 2D square. The quality of this method depends on how well the underlying surface can be represented by a sphere, which we will explore in Section~\ref{sec:autoencoding}.

\section{Results}
In this section we show qualitative and quantitative results on the tasks of auto-encoding 3D shapes and single-view reconstruction and compare against several baselines. 
In addition to these tasks, we also demonstrate several additional applications of our approach. 
More results are available in the supplementary material \cite{appendix}.%

\myparagraph{Data.} 
We evaluated our approach on the standard ShapeNet Core dataset (v2) \cite{shapenet2015}.
The dataset consists of 3D models covering 13 object categories with 1K-10K shapes per category. 
We used the training and validation split provided by \cite{choy20163d} for our experiments to be comparable with previous approaches.
We used the rendered views provided by  \cite{choy20163d} and sampled 3D points on the shapes using \cite{Wang-2017-OCNN}.

\myparagraph{Evaluation criteria.}
We evaluated our generated shape outputs by comparing to ground truth shapes using two criteria. 
First, we compared point sets for the output and ground-truth shapes using Chamfer distance (``CD''). 
While this criteria compares two point sets, it does not take into account the surface/mesh connectivity. 
To account for mesh connectivity, we compared the output and ground-truth meshes using the ``Metro'' criteria using the publicly available METRO software \cite{cignoni1998metro}, which is the average Euclidean distance between the two meshes. 

\myparagraph{Points baseline.}
In addition to existing baselines, we compare our approach to the multi-layer perceptron ``Points baseline'' network shown in Figure \ref{fig:baseline}. 
The Points baseline network consists of four fully connected layers with output dimensions of size 1024, 512, 256, 7500 with ReLU non-linearities, batch normalization on the first three layers, and a hyperbolic-tangent non-linearity after the final fully connected layer. The network outputs 2500 3D points and has comparable number of parameters to our method with 25 learned parameterizations. The baseline architecture was designed to be as close as possible to the MLP used in \ournet{}. 
As the network outputs points and not a mesh, we also trained a second network that outputs 3D points and normals, which are then passed as inputs to Poisson surface reconstruction (PSR) ~\cite{kazhdan2013screened} to generate a mesh (``Points baseline + normals''). 
The network generates outputs in $\mathbb{R}^6$ representing both the 3D spatial position and normal. 
We optimized Chamfer loss in this six-dimensional space and normalized the normals to 0.1 length as we found this trade-off between the spatial coordinates and normals in the loss worked best. 
As density is crucial to PSR quality, we augmented the number of points by sampling 20 points in a small radius in the tangent plane around each point ~\cite{kazhdan2013screened}.  We noticed significant qualitative and quantitative improvements and the results shown in this paper use this augmentation scheme.

\begin{figure*}[t!]
\centering
\begin{subfigure}[b]{0.12\linewidth}
\centering
 \includegraphics[height=0.6\linewidth]{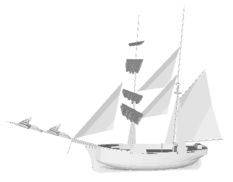}
 \includegraphics[height=0.45\linewidth]{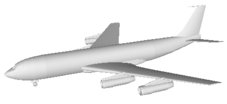}
 \includegraphics[height=0.65\linewidth]{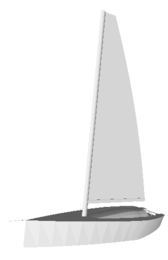}
 \includegraphics[height=0.60\linewidth]{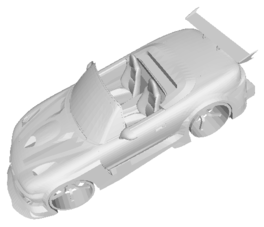}
 \includegraphics[height=0.7\linewidth]{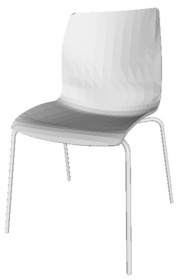}
\caption{Ground truth}
\end{subfigure}
\begin{subfigure}[b]{0.12\linewidth}
\centering
 \includegraphics[height=0.6\linewidth]{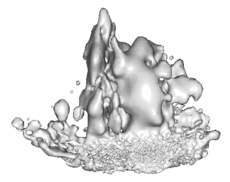}
 \includegraphics[height=0.45\linewidth]{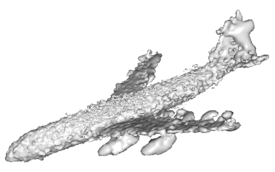}
 \includegraphics[height=0.65\linewidth]{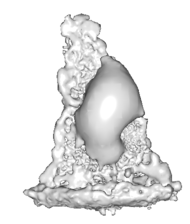}
 \includegraphics[height=0.60\linewidth]{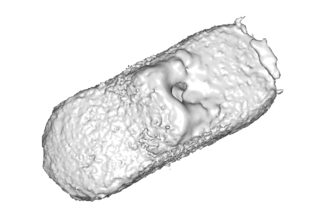}
 \includegraphics[height=0.7\linewidth]{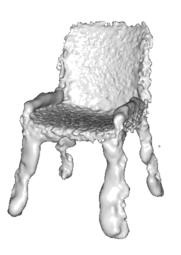}
\caption{Pts baseline}
\label{fig:ae_comparison_base}
\end{subfigure}
\begin{subfigure}[b]{0.12\linewidth}
\centering
 \includegraphics[height=0.6\linewidth]{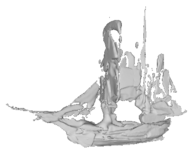}
 \includegraphics[height=0.45\linewidth]{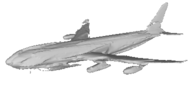}
 \includegraphics[height=0.65\linewidth]{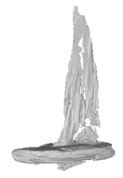}
 \includegraphics[height=0.60\linewidth]{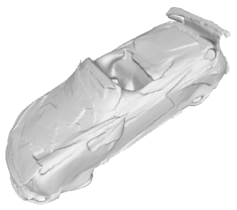}
 \includegraphics[height=0.7\linewidth]{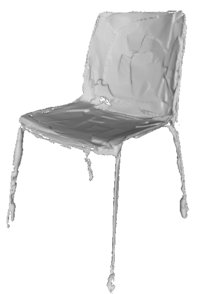}
\caption{PSR on ours}
\label{fig:ae_comparison_psr}
\end{subfigure}
\begin{subfigure}[b]{0.12\linewidth}
\centering
 \includegraphics[height=0.6\linewidth]{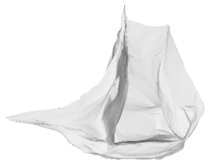}
 \includegraphics[height=0.45\linewidth]{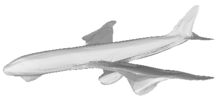}
 \includegraphics[height=0.65\linewidth]{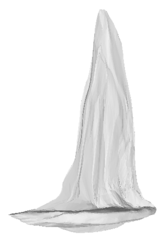}
 \includegraphics[height=0.60\linewidth]{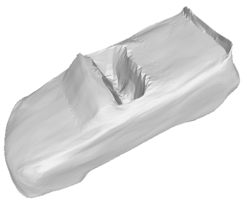}
 \includegraphics[height=0.7\linewidth]{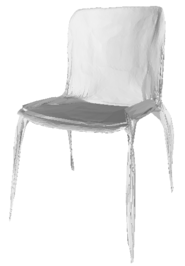}
\caption{Ours sphere}
\end{subfigure}
\begin{subfigure}[b]{0.12\linewidth}
\centering
 \includegraphics[height=0.6\linewidth]{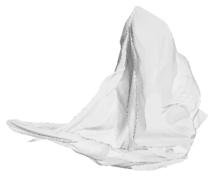}
 \includegraphics[height=0.45\linewidth]{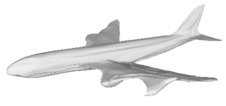}
 \includegraphics[height=0.65\linewidth]{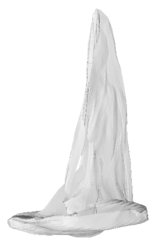}
 \includegraphics[height=0.60\linewidth]{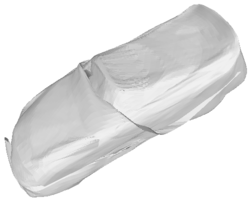}
 \includegraphics[height=0.7\linewidth]{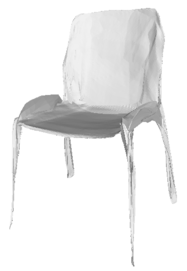}
\caption{Ours 1}
\end{subfigure}
\begin{subfigure}[b]{0.12\linewidth}
\centering
 \includegraphics[height=0.6\linewidth]{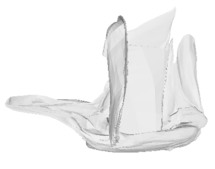}
 \includegraphics[height=0.45\linewidth]{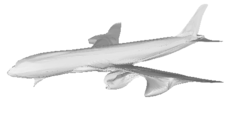}
 \includegraphics[height=0.65\linewidth]{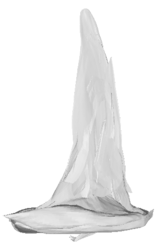}
 \includegraphics[height=0.60\linewidth]{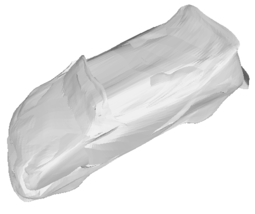}
 \includegraphics[height=0.7\linewidth]{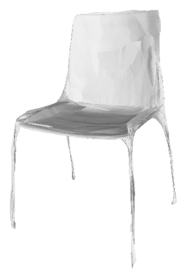}
\caption{Ours 5}
\end{subfigure}
\begin{subfigure}[b]{0.12\linewidth}
\centering
 \includegraphics[height=0.6\linewidth]{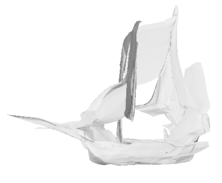}
 \includegraphics[height=0.45\linewidth]{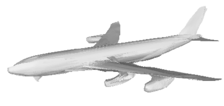}
 \includegraphics[height=0.65\linewidth]{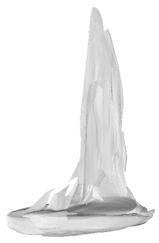}
 \includegraphics[height=0.60\linewidth]{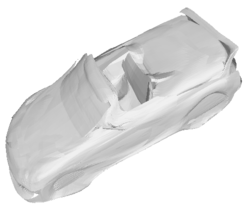}
 \includegraphics[height=0.7\linewidth]{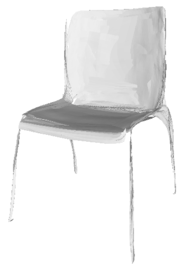}
\caption{Ours 25}
\end{subfigure}
\begin{subfigure}[b]{0.12\linewidth}
\centering
 \includegraphics[height=0.6\linewidth]{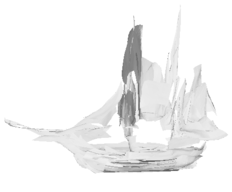}
 \includegraphics[height=0.45\linewidth]{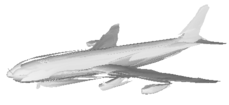}
 \includegraphics[height=0.65\linewidth]{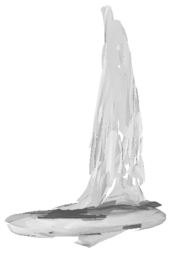}
 \includegraphics[height=0.60\linewidth]{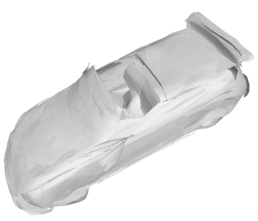}
 \includegraphics[height=0.7\linewidth]{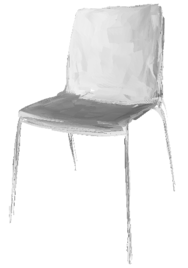}
\caption{Ours 125}
\end{subfigure}
\caption{
{\bf Auto-encoder.}
We compare the original meshes (a) to meshes obtained by running PSR on the point clouds generated by the baseline (b) and on the densely sampled point cloud from our generated mesh (c),  and to our method generating a surface from a sphere (d), 1 (e), 5 (f), 25 (g), and 125(h) learnable parameterizations. Notice the fine details in (g) and (h) : e.g. the plane's engine and the jib of the ship. %
}
  \label{fig:ae_comparison}
  \vspace{-3mm}

\end{figure*}

\subsection{Auto-encoding 3D shapes}
\label{sec:autoencoding}

In this section we evaluate our approach to generate a shape given an input 3D point cloud and compare against the Points baseline.
We evaluate how well our approach can generate the shape, how it can generalize to object categories not seen during training, and its sensitivity to the number of patches.

\begin{table}[t!]
\centering
{
  \begin{tabular}{l|c|c}
  \hline
  Method & CD & Metro    \\

  \hline
  \hline
 Oracle 2500 pts & 0.85 & 1.56 \\
 Oracle 125K pts & - & 1.26 \\
\hline
 Points baseline & {1.91}  & -  \\
 Points baseline + normals & {2.15} & {1.82} (PSR)   \\

 Ours - 1 patch & {1.84} &  {1.53}  \\
 Ours - 1 sphere & {1.72}  & {1.52}  \\
 Ours - 5 patches & {1.57} &  {1.48}  \\
 Ours - 25 patches & {1.56} & {1.47}  \\
 Ours - 125 patches & \textbf{1.51} &  \textbf{1.41}  \\

  \end{tabular}
  }
      \caption{{\bf 3D reconstruction.}
      Comparison of our approach against a point-generation baseline (``CD'' - Chamfer distance, multiplied by $10^3$, computed on 2500 points; ``Metro'' values are multiplied by 10). Note that our approach can be directly evaluated by Metro while the baseline requires performing PSR ~\cite{kazhdan2013screened}. These results can be compared with an Oracle sampling points directly from the ground truth 3D shape followed by PSR (top two rows). See text for details.
      }\label{tab:compare_category}
        \vspace{-3mm}

\end{table}

\myparagraph{Evaluation on surface generation.}
We report quantitative results for shape generation from point clouds in Table~\ref{tab:compare_category}, where each approach is trained over all ShapeNet categories and results are averaged over all categories. 
Notice that our approach out-performs the Points baseline on both the Chamfer distance and Metro criteria, even when using a single learned parameterization (patch). %
Also, the Points baseline + normals has worse Chamfer distance than the Points baseline without normals indicating that predicting the normals decreases the quality of the point cloud generation. 

We also report performance for two ``oracle'' outputs indicating upper bounds in Table~\ref{tab:compare_category}. 
The first oracle (``Oracle 2500 pts'') randomly samples 2500 points+normals from the ground truth shape and applies PSR. 
The Chamfer distance between the random point set and the ground truth gives an upper bound on performance for point-cloud generation. 
Notice that our method out-performs the surface generated from the oracle points. 
The second oracle (``Oracle 125K pts'') applies PSR on all 125K points+normals from the ground-truth shape. 
It is interesting to note that the Metro distance from this result to the ground truth is not far from the one obtained with our method.

We show qualitative comparisons in Figure~\ref{fig:ae_comparison}.
Notice that the PSR from the baseline point clouds (Figure~\ref{fig:ae_comparison_base}) look noisy and lower quality than the meshes produced directly by our method and PSR performed on points generated from our method as described in Section~\ref{sec:mesh} (Figure~\ref{fig:ae_comparison_psr}). %
\myparagraph{Sensitivity to number of patches.}
We show in Table~\ref{tab:compare_category} our approach with varying number of learnable parameterizations (patches) in the atlas.
Notice how our approach improves as we increase the number of patches.
Moreover, we also compare with the approach described in Section \ref{sec:mesh} which samples points on the 3D unit sphere instead of 2D patches to obtain a closed mesh. %
Notice that sampling from a sphere quantitatively out-performs a single patch, but multiple patches perform better.

We show qualitative results for varying number of learnable parameterizations in Figure~\ref{fig:ae_comparison}. As suggested by the quantitative results, the visual quality improves with the number of parameterizations. 
However, more artifacts appear with more parameterizations, such as close-but-disconnected patches (e.g., sail of the sailboat) . We thus used 25 patches for the single-view reconstruction experiments (Section \ref{sec:svr})

\begin{table}[t!]
\centering
{
  \begin{tabular}{l|c|c|c|c}
  \hline
  \multicolumn{2}{c|}{Category} & Points  & Ours & Ours   \\
               & &                 baseline &  1 patch &  125 patches  \\
  \hline
  \hline
  \multirow{2}{*}{chair} & LOO &  {3.66} & {3.43} & \textbf{2.69}  \\
   & All &  {1.88} & {1.97}   & \textbf{1.55}\\
     \hline

  \multirow{2}{*}{car} & LOO & {3.38} & {2.96} & \textbf{2.49}   \\
   & All & {1.59} & {2.28} & \textbf{1.56}   \\
     \hline

  \multirow{2}{*}{watercraft} & LOO & {2.90} & {2.61} & \textbf{1.81}   \\
   & All & {1.69} & {1.69} &\textbf{1.23}   \\
     \hline

  \multirow{2}{*}{plane} & LOO & {6.47} & {6.15} & \textbf{3.58} \\
   & All & {1.11} & {1.04} & \textbf{0.86} \\

  \end{tabular} 
  }
  \caption{
  {\bf Generalization across object categories.}
  Comparison of our approach with varying number of patches against the point-generating baseline to generate a specific category when training on all other ShapeNet categories.
  Chamfer distance is reported, multiplied by $10^3$, computed on 2500 points.
  Notice that our approach with 125 patches out-performs all baselines when generalizing to the new category.
  For reference, we also show performance when we train over all categories.
  }\label{tab:gen_category}
    \vspace{-3mm}
\end{table}
\begin{figure}[t!]
\centering
\begin{subfigure}[b]{0.49\linewidth}
\centering
 \includegraphics[width=0.27\linewidth]{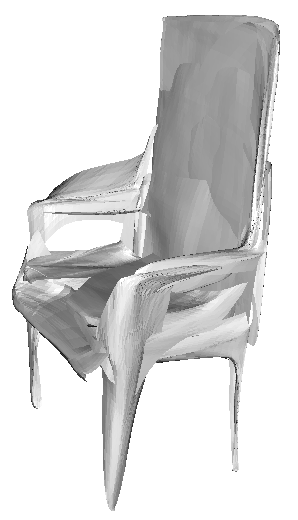}
 \includegraphics[width=0.25\linewidth]{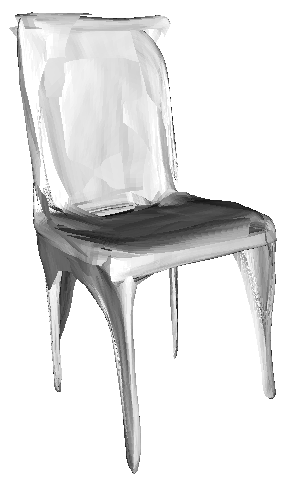}
 \includegraphics[width=0.4\linewidth]{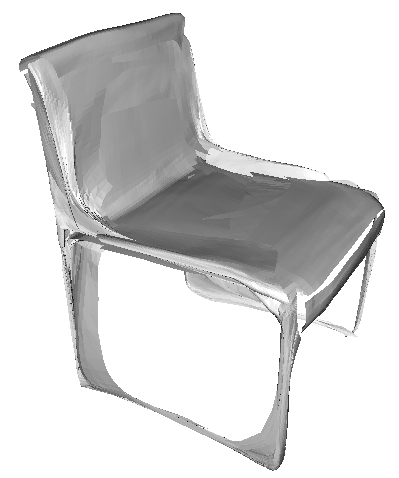}
\caption{Not trained on chairs}
\end{subfigure}
\vline
\begin{subfigure}[b]{0.49\linewidth}
\centering
 \includegraphics[width=0.28\linewidth]{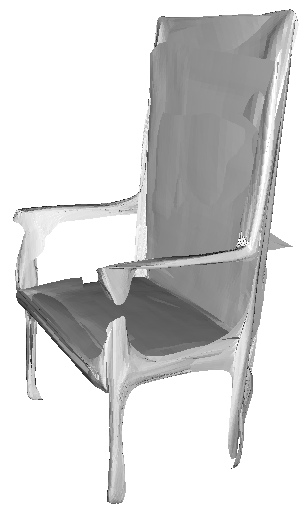}
 \includegraphics[width=0.28\linewidth]{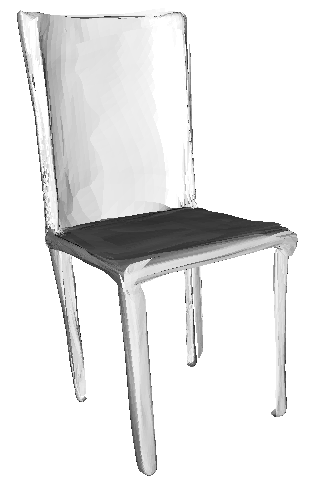}
 \includegraphics[width=0.38\linewidth]{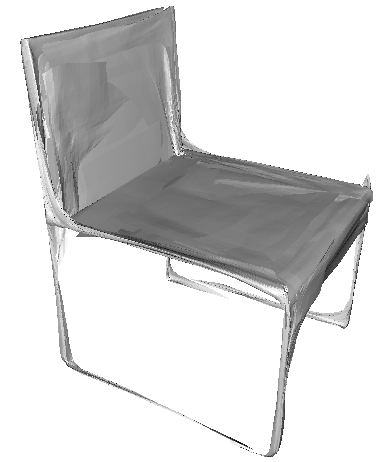}
\caption{Trained on all categories}
\end{subfigure}
\caption{
{\bf Generalization.} (a) Our method (25 patches) can generate surfaces close to a category never seen during training. It, however, has more artifacts than if it has seen the category during training (b), e.g., thin legs and armrests.
}
  \label{fig:gen}
    \vspace{-3mm}
\end{figure}

\myparagraph{Generalization across object categories.} %
An important desired property of a shape auto-encoder is that it generalizes well to categories it has not been trained on. To evaluate this, we trained our method on all categories but one target category (``LOO'') for chair, car, watercraft, and plane categories, and evaluated on the held-out category. The corresponding results are reported in Table~\ref{tab:gen_category} and Figure \ref{fig:gen}. We also include performance when the methods are trained on all of the categories including the target category (``All'') for comparison.
Notice that we again out-perform the point-generating baseline on this leave-one-out experiment and that performance improves with more patches. The car category is especially interesting since when trained on all categories the baseline has better results than our method with 1 patch and similar to our method with 125 patches. If not trained on cars, both our approaches clearly outperform the baseline, showing that at least in this case, our approach generalizes better than the baseline. %
The visual comparison shown Figure \ref{fig:gen} gives an intuitive understanding of the consequences of not training for a specific category. When not trained on chairs, our method seems to struggle to define clear thin structures, like legs or armrests, especially when they are associated to a change in the topological genus of the surface. This is expected as these types of structures are not often present in the categories the network was trained on.

\subsection{Single-view reconstruction}
\label{sec:svr}
\captionsetup[table]{skip=4mm}

\begin{table*}[t!]
\centering
{
\small
  \begin{tabular}{l|c|c|c|c|c|c|c|c|c|c|c|c|c|c}
   &  pla. &  ben. &  cab. &  car &  cha. &  mon. &  lam. &  spe. &  fir. &  cou. &  tab. &  cel. &  wat. &  mean \\
  \hline
 Ba CD & 2.91 & 4.39 & 6.01 & 4.45 & 7.24 & \textbf{5.95} & 7.42 & 10.4 & 1.83 & \textbf{6.65} & 4.83 & 4.66 & \textbf{4.65} & 5.50 \\
 {PSG CD} & 3.36 & 4.31 & 8.51 & 8.63 & \textbf{6.35} & 6.47 & 7.66 & 15.9 & \textbf{1.58} & 6.92 & \textbf{3.93} & \textbf{3.76} & 5.94 & 6.41  \\
 Ours CD & \textbf{2.54} & \textbf{3.91} & \textbf{5.39} & \textbf{4.18} & 6.77 & {6.71} & \textbf{7.24} & \textbf{8.18} & 1.63 & 6.76 & 4.35 & 3.91 & 4.91 & \textbf{5.11}  \\
 \hline
 \hline
 Ours Metro& 1.31 & 1.89 & 1.80 & 2.04 & 2.11 & 1.68 & 2.81 & 2.39 & 1.57 & 1.78 & 2.28 & 1.03 & 1.84 & 1.89 \\
  \end{tabular}
  }
\medskip
\caption{\textbf{Single-View Reconstruction (per category).} The mean is taken category-wise. The Chamfer Distance reported is computed on 1024 points, after running ICP alignment with the GT point cloud, and multiplied by $10^3$. The Metro distance is multiplied by 10.
  }
  \label{tab:SVR}
  \vspace{1mm}
\end{table*}

\setlength{\belowcaptionskip}{-1pt}

\begin{figure}[t!]

\centering
\begin{subfigure}[b]{0.19\linewidth}
\centering
 \includegraphics[height=0.65\linewidth]{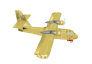}\\
 \includegraphics[height=0.65\linewidth]{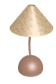}\\
 \includegraphics[height=0.65\linewidth]{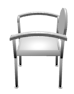}\\
 \includegraphics[height=0.65\linewidth]{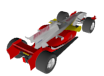}\\
 \includegraphics[height=0.65\linewidth]{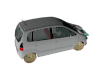}
\caption{Input}
\end{subfigure}
\begin{subfigure}[b]{0.19\linewidth}
\centering
 \includegraphics[height=0.55\linewidth]{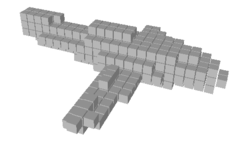}\\[0.1\linewidth]
 ~~~~\includegraphics[height=0.55\linewidth]{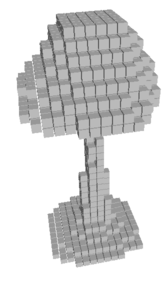}~~~~\\[0.1\linewidth]
 \includegraphics[height=0.55\linewidth]{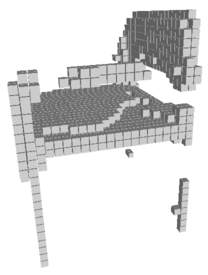}\\[0.1\linewidth]
 \includegraphics[height=0.55\linewidth]{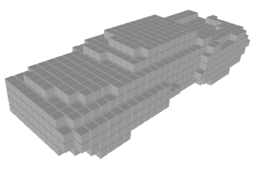}\\[0.1\linewidth]
 \includegraphics[height=0.55\linewidth]{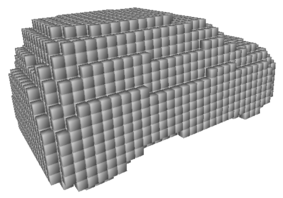}\\[0.05\linewidth]
\caption{{\footnotesize 3D-R2N2}}
\end{subfigure}
\begin{subfigure}[b]{0.19\linewidth}
\centering
 \includegraphics[height=0.55\linewidth]{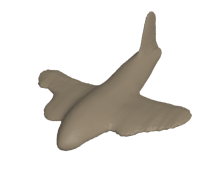}\\[0.05\linewidth]
 \includegraphics[height=0.55\linewidth]{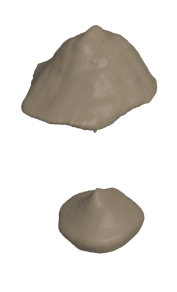}\\[0.1\linewidth]
 \includegraphics[height=0.55\linewidth]{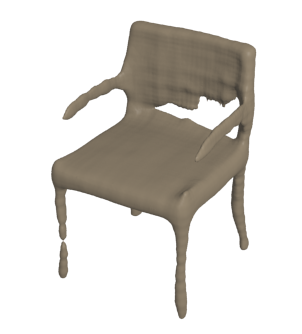}\\[0.1\linewidth]
 ~~~~\includegraphics[height=0.55\linewidth]{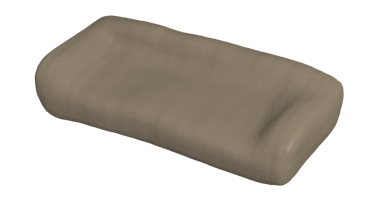}~~~~\\[0.1\linewidth]
 \includegraphics[height=0.55\linewidth]{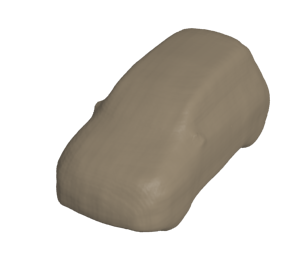}\\[0.1\linewidth]
\caption{HSP}
\end{subfigure}
\begin{subfigure}[b]{0.19\linewidth}
\centering
 \includegraphics[height=0.55\linewidth]{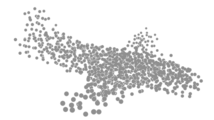}\\[0.1\linewidth]
 ~~~~\includegraphics[height=0.65\linewidth]{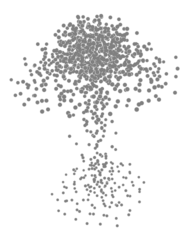}~~~~\\[0.1\linewidth]
 \includegraphics[height=0.55\linewidth]{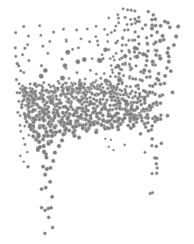}\\[0.1\linewidth]
 \includegraphics[height=0.55\linewidth]{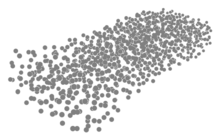}\\[0.1\linewidth]
 \includegraphics[height=0.55\linewidth]{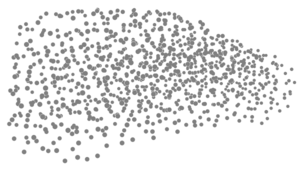}\\[0.05\linewidth]
\caption{PSG}
\end{subfigure}
\begin{subfigure}[b]{0.19\linewidth}
\centering
 \includegraphics[height=0.55\linewidth]{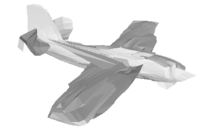}\\[0.1\linewidth]
 ~~~~\includegraphics[height=0.55\linewidth]{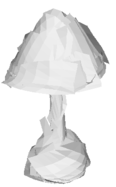}~~~~\\[0.1\linewidth]
 \includegraphics[height=0.55\linewidth]{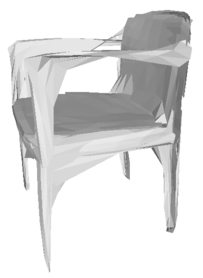}\\[0.1\linewidth]
 \includegraphics[height=0.55\linewidth]{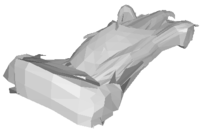}\\[0.1\linewidth]
 \includegraphics[height=0.55\linewidth]{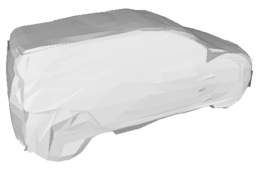}\\[0.05\linewidth]
\caption{Ours}
\end{subfigure}
\caption{
{\bf Single-view reconstruction comparison.}
From a 2D RGB image (a), 3D-R2N2 \cite{choy20163d} reconstructs a voxel-based 3D model (b),  HSP \cite{Hane:2017} reconstructs a octree-based 3D model (c), PointSetGen \cite{Fan:2017:cvpr} a point cloud based 3D model (d), and our \ournet{} a triangular mesh (e).}
 \label{fig:svr_comparison}
 \vspace{-4mm}
\end{figure}

\begin{figure}[t!]
\centering
\begin{subfigure}[b]{0.30\linewidth}
\begin{minipage}[c]{0.99\linewidth}
\centering
 \includegraphics[height=0.55\linewidth, width=0.55\linewidth]{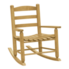}\\
 \includegraphics[height=0.55\linewidth, width=0.55\linewidth]{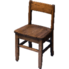}\\
  \includegraphics[height=0.55\linewidth, width=0.55\linewidth]{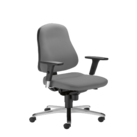}\\
\caption{Input}
\end{minipage}
\end{subfigure}
\begin{subfigure}[b]{0.30\linewidth}
\begin{minipage}[c]{0.99\linewidth}

\centering
 \includegraphics[height=0.55\linewidth, width=0.55\linewidth]{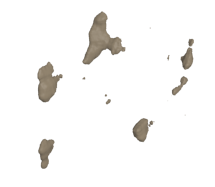}\\
 \includegraphics[height=0.55\linewidth, width=0.55\linewidth]{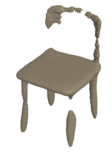}\\
  \includegraphics[height=0.55\linewidth, width=0.55\linewidth]{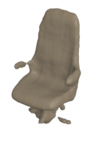}
\caption{HSP}
\end{minipage}

\end{subfigure}
\begin{subfigure}[b]{0.30\linewidth}
\begin{minipage}[c]{0.99\linewidth}

\centering
 \includegraphics[height=0.55\linewidth, width=0.55\linewidth]{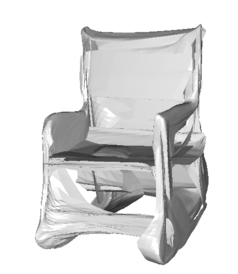}\\
 \includegraphics[height=0.55\linewidth, width=0.55\linewidth]{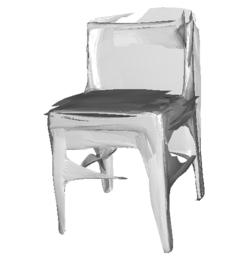}\\
 \includegraphics[height=0.55\linewidth, width=0.55\linewidth]{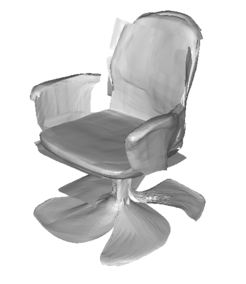}\\

\caption{Ours}
\end{minipage}

\end{subfigure}
\setlength{\belowcaptionskip}{-5pt}
\medskip
\caption{
{\bf Single-view reconstruction comparison on natural images.}
From a 2D RGB image taken from internet (a),  HSP \cite{Hane:2017} reconstructs a octree-based 3D model (b), and our \ournet{} a triangular mesh (c).}
  \label{fig:svr_real_images} \vspace{-4mm}
\end{figure}

We evaluate the potential of our method for single-view reconstruction. We compare qualitatively our results with three state-of-the-art methods, PointSetGen \cite{Fan:2017:cvpr}, 3D-R2N2 \cite{choy20163d} and HSP \cite{Hane:2017} in Figure~\ref{fig:svr_comparison}. To perform the comparison for PointSetGen \cite{Fan:2017:cvpr} and 3D-R2N2 \cite{choy20163d}, we used the trained models made available online by the authors. %
For HSP \cite{Hane:2017}, we asked the authors to run their method on the images in Fig.\ \ref{fig:svr_comparison}. 
Note that since their model was trained on images generated with a different renderer, this comparison is not absolutely fair. To remove the bias we also compared our results with HSP on real images for which none of the methods was trained (Fig.\ \ref{fig:svr_real_images}) which also demonstrates the ability of our network to generalize to real images.

Figure \ref{fig:svr_comparison} emphasizes the importance of the type of output (voxels for 3D-N2D2 and HSP, point cloud for PointSetGen, mesh for us) for the visual appearance of the results. Notice the small details visible on our meshes that may be hard to see on the unstructured point cloud or volumetric representation. Also, it is interesting to see that PointSetGen tends to generate points inside the volume of the 3D shape while our result, by construction, generates points on a surface.

To perform a quantitative comparison against PointSetGen \cite{Fan:2017:cvpr}, we evaluated the Chamfer distance between generated points and points from the original mesh for both PointSetGen and our method with 25 learned parameterizations. However, the PointSetGen network was trained with a translated, rotated, and scaled version of ShapeNet with parameters we did not have access to. We thus first had to align the point clouds resulting from PointSetGen to the ShapeNet models used by our algorithm.
We randomly selected 260 shapes, 20 from each category, and ran the iterative closest point (ICP) algorithm~\cite{besl1992method} to optimize a similarity transform between PointSetGen and the target point cloud. Note that this optimization improves the Chamfer distance between the resulting point clouds, but is not globally convergent. We checked visually that the point clouds from PointSetGen were correctly aligned, and display all alignments on the project webpage\footnote{\url{http://imagine.enpc.fr/~groueixt/atlasnet/PSG.html}.}. To have a fair comparison we ran the same ICP alignment on our results. In Table \ref{tab:SVR} we compared the resulting Chamfer distance. Our method provides the best results on 6 categories whereas PointSetGen and the baseline are best on 4 and 3 categories, respectively. Our method is better on average and generates point clouds of a quality similar to the state of the art. We also report the Metro distance to the original shape, which is the most meaningful measure for our method.

To quantitatively compare against HSP \cite{Hane:2017}, we retrained our method on their publicly available data since train/test splits are different from 3D-R2N2 \cite{choy20163d} and they made their own renderings of ShapeNet data. Results are in Table \ref{tab:hsp_compare}. More details are in the supplementary \cite{appendix}.
\setlength{\belowcaptionskip}{2pt}

\begin{table}[h!]
\centering
 \begin{tabular}{l|c|c|c|c|c|c}
 & \multicolumn{1}{|l|}{Chamfer} & \multicolumn{1}{|l|}{Metro} \\
\hline
HSP  \cite{Hane:2017} & 11.6 &  1.49\\
Ours  (25 patches)  & \textbf{9.52} &  \textbf{1.09} \\
 \end{tabular}
 \medskip
  \caption{{\bf Single-view reconstruction.} Quantitative comparison against HSP \cite{Hane:2017}, a state of the art octree-based method. The average error is reported, on 100 shapes from each category. The Chamfer Distance reported is computed on $10^4$ points, and multiplied by $10^3$. The Metro distance is multiplied by 10. More details are in the supplemetary \cite{appendix}.}\label{tab:hsp_compare}
    \vspace{-3mm}
\end{table}

\begin{figure*}[t!]
\centering
\begin{subfigure}[b]{0.32\linewidth}
 \includegraphics[width=\linewidth]{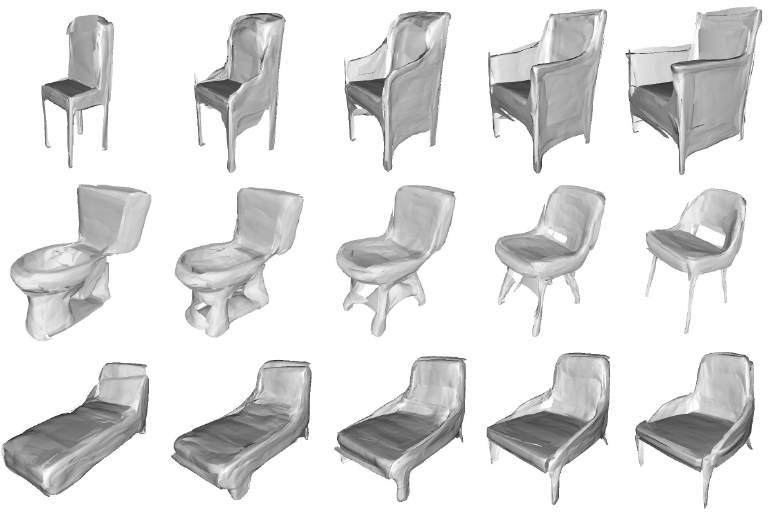}
\caption{Shape interpolation. \label{fig:interpolation}}
\end{subfigure}
~\vline~
\begin{subfigure}[b]{0.32\linewidth}
 \includegraphics[width=\linewidth]{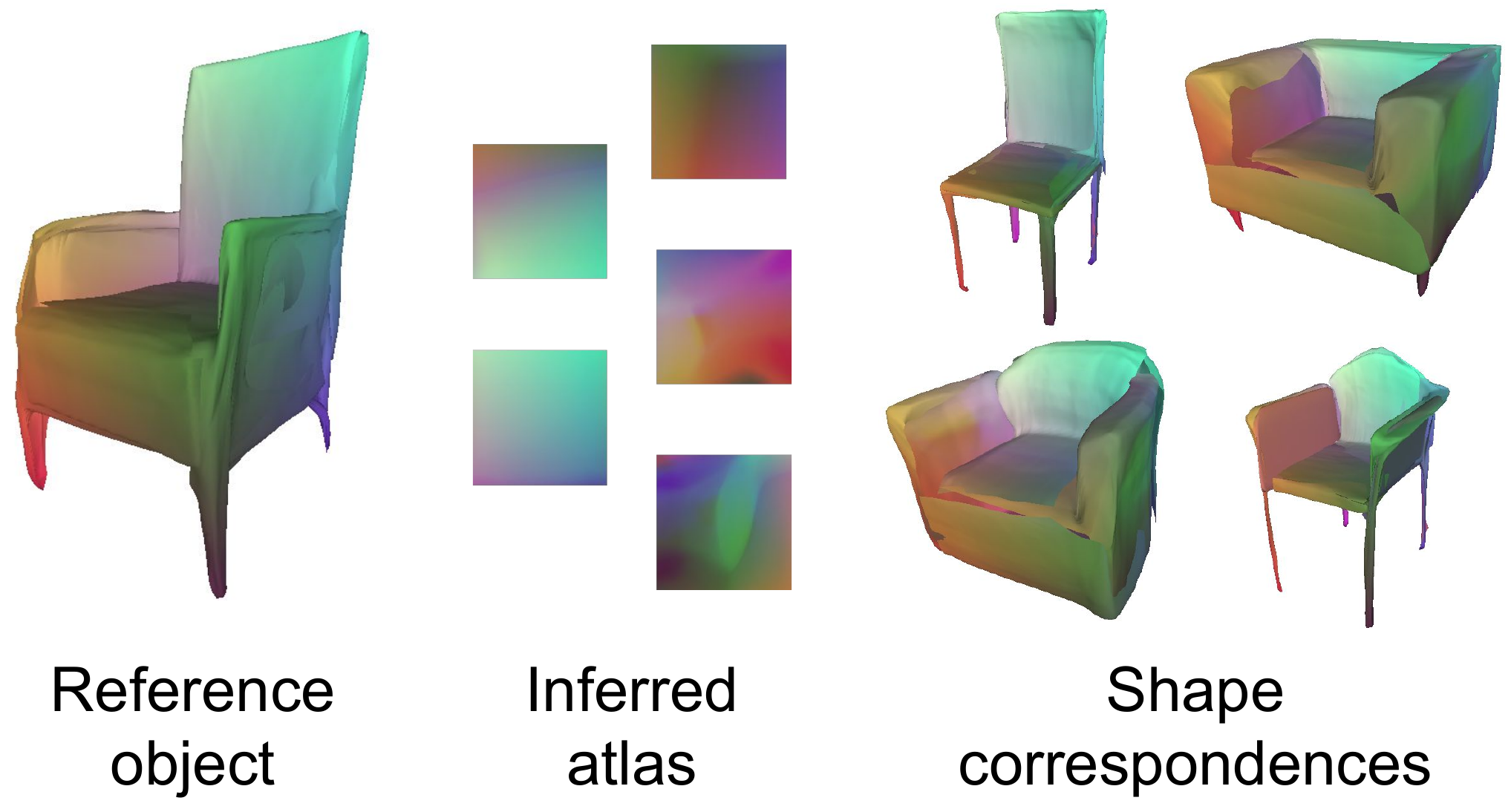}
\caption{Shape correspondences. \label{fig:corresp}}
\end{subfigure}
~\vline~
\begin{subfigure}[b]{0.32\linewidth}
 \includegraphics[width=\linewidth]{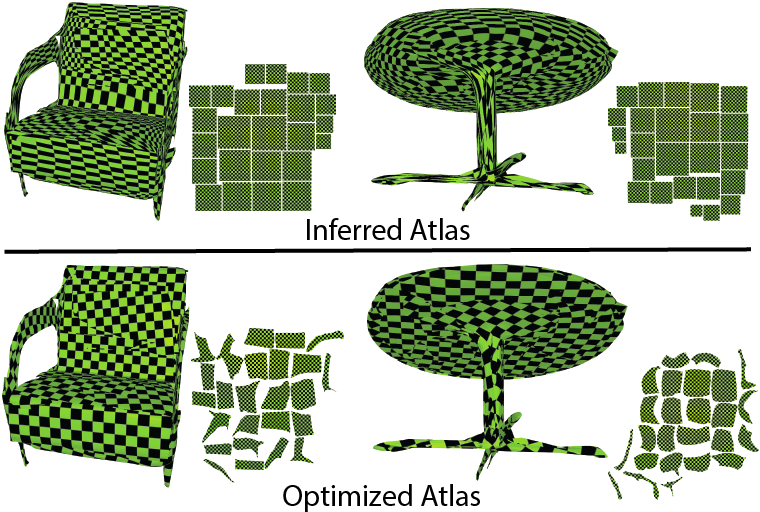}
\caption{Mesh parameterization. \label{fig:mesh_parameterization}}
\end{subfigure}
\caption{{\bf Applications.} Results from three applications of our method. See text for details.
}
\vspace{-10pt}

  \label{fig:applications}
\end{figure*}

\subsection{Additional applications}

\myparagraph{Shape interpolation.}
Figure \ref{fig:interpolation} shows shape interpolation.
Each row shows interpolated shapes generated by our \ournet{}, starting from the shape in the first column to the shape in the last. Each intermediate shape is generated using a weighted sum of the latent representations of the two extreme shaped.
Notice how the interpolated shapes gradually add armrests in the first row, and chair legs in the last.

\myparagraph{Finding shape correspondences.}
Figure~\ref{fig:corresp} shows shape correspondences.
We colored the surface of reference chair (left) according to its 3D position.
We transfer the surface colors from the reference shape to the inferred atlas (middle).
Finally, we transfer the atlas colors to other shapes (right) such that points with the same color are parametrized by the same point in the atlas.
Notice that we get semantically meaningful correspondences, such as the chair back, seat, and legs without any supervision from the dataset on semantic information.

\myparagraph{Mesh parameterization}
Most existing rendering pipelines require an atlas for texturing a shape (Figure~\ref{fig:mesh_parameterization}). A good parameterization should minimize
amount of area distortion ($E_a$) and stretch ($E_s$) of a UV map.
We computed average per-triangle distortions for 20 random shapes from each category and found that our inferred atlas usually has relatively high texture distortion ($E_a\!=\!1.9004, E_s\!=\!6.1613$, where undistorted map has $E_a\!=\!E_s\!=\!1$).
Our result, however, is well-suited for distortion minimization because all meshes have disk-like topology and inferred map is bijective, making it easy to further minimize distortion with off-the-shelf geometric optimization~\cite{AQP:2016}, yielding small distortion ($E_a\!=\!1.0016, E_s\!=\!1.025$, see bottom row for example).

\myparagraph{Limitations and future work} are detailed in the supplementary materials \cite{appendix}.
 
\section{Conclusion}
We have introduced an approach to generate parametric surface elements for 3D shapes. 
We have shown its benefits for 3D shape and single-view reconstruction, out-performing existing baselines. 
In addition, we have shown its promises for shape interpolation, finding shape correspondences, and mesh parameterization. 
Our approach opens up applications in generation and synthesis of meshes for 3D shapes, similar to still image generation~\cite{Isola:2017,Zhu:2017:cylcegan}.

\myparagraph{Acknowledgments.} This work was partly supported by ANR project EnHerit ANR-17-CE23-0008, Labex B\'ezout, and gifts from Adobe to \'Ecole des Ponts. We thank Aaron Herzmann for fruitful discussions, Christian H\"ane for his help in comparing to \cite{Hane:2017}
and Kevin Wampler for helping with geometric optimization for surface parameterization. 

{\small
\bibliographystyle{ieee}
\bibliography{egbib}
}
\clearpage
\section{Supplementary}
\subsection{Overview}
This document provides more detailed quantitative and qualitative results highlighting the strengths and limitations of \ournet{}. \\
\vspace*{-3pt}

\paragraph{Detailed results, per category, for the autoencoder} These tables report the metro reconstruction error and the chamfer distance error. It surprisingly shows that our method with 25 learned parameterizations outperforms our method with 125 learned parameterizations in 7 categories out of 13 for the metro distance, but is significantly worse on the cellphone category, resulting in the 125 learned parameterizations approach being better on average. This is not mirrored in the Chamfer distance.
\vspace*{-2pt}

\paragraph{Regularisation} In  the  autoencoder experiment, we tried using weight decay with different weight. The best results were obtained without any regularization.
\vspace*{-2pt}

\paragraph{Limitations} We describe two limitations with our approach. First, when a small number of learned parameterizations are used, the network has to distort them too much to recreate the object. This leads, when we try to recreate a mesh, to small triangles in the learned parameterization space being distorted and become large triangles in 3D covering undesired regions. On the other hand, as the number of  learned parameterization increases, errors in the topology of the reconstructed mesh can be sometimes observed. 
In practice, it means that the reconstructed patches overlap, or are not stiched together.
\vspace*{-2pt}

\paragraph{Additional Single View Reconstruction qualitative results} In this figure, we show one example of single-view reconstruction per category and compare with the state of the art, PointSetGen and 3D-R2N2. We consistently show that our method produces a better reconstruction.
\vspace*{-2pt}

\paragraph{Additional Autoencoder qualitative results}
In this figure, we show one example per category of autoencoder reconstruction for the baseline and our various approaches to reconstruct meshes, detailed in the main paper. We show how we are able to recreate fine surfaces.
\vspace*{-2pt}

\paragraph{Additional Shape Correspondences qualitative results} We color each vertex of the reference object by its distance to the gravity center of the object, and transfer these colors to the inferred atlas. We then propagate them to other objects of the same category, showing semantically meaningful correspondences between them. Results for the plane and watercraft categories are shown and generalize to all categories.

\paragraph{Deformable shapes.}
We ran an experiment on human shape to show that our method is also suitable for reconstructing deformable shapes. The FAUST dataset~\cite{Bogo:CVPR:2014} is a collection of meshes representing several humans in different poses. We used 250 shapes for training, and 50 for validation (without using the ground truth correspondences in any way). In table \ref{tab:faust}, we report the reconstruction error in term of Chamfer distance and Metro distance for our method with 25 squarred parameterizations, our methods with a sphere parametrization, and for the baseline. We found results to be consistent with the analysis on ShapeNet. Qualitative results are shown in figure~\ref{fig:faust}, revealing that our method leads to qualitatively good reconstructions. %

\begin{table}[h!]
\centering
 \begin{tabular}{l|c|c|c|c|c|c}
 & \multicolumn{1}{|l|}{Chamfer} & \multicolumn{1}{|l|}{Metro} \\
\hline
25 patches  & 15.47 &  11.62 \\
1 Sphere  & 15.78 &  15.22 \\
1 Ref. Human & 16.39 & 13.46 
 \end{tabular}

  \caption{{\bf 3D Reconstruction on FAUST~\cite{Bogo:CVPR:2014}.} We trained the baseline and our method sampling the points according from 25 square patches, and from a sphere on the human shapes from the FAUST dataset. We report Chamfer distance (x $10^4$) on the points and Metro distance (x10) on the meshes. }\label{tab:faust}
    \vspace{-3mm}
\end{table}

\paragraph{Point cloud super-resolution}
AtlasNet can generate pointclouds or meshes of arbitrary resolution simply by sampling more points.
Figure ~\ref{fig:super_resolution} shows qualitative results of our approach with 25 patches generating high resolution meshes with 122500 points. Moreover,  PointNet is able to take an arbitrary number of points as input and encodes a minimal shape based on a subset of the input points. This is a double-edged sword : while it allows the autoencoder to work with varying number of input points, it also prevent it from reconstructing very fine details, as they are not used by PointNet and thus not present in the latent code. We show good results using only 250 input points, despite the fact that we train using 2500 input points which shows the capacity of our decoder to interpolate a surface from a small number of input points, and the flexibility of our pipeline.

\begin{figure*}
\centering
\begin{minipage}[c]{0.09\linewidth}
    2500 \\ points \\ \\ \\ \\
    250  \\ points
\end{minipage}
\begin{minipage}[c]{0.42\linewidth}
  \begin{subfigure}[b]{0.99\linewidth}
  \centering
   \includegraphics[height=0.55\linewidth, width=0.45\linewidth]{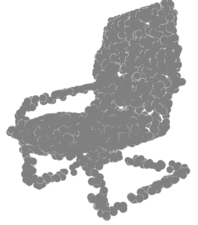}\\
   \includegraphics[height=0.55\linewidth, width=0.45\linewidth]{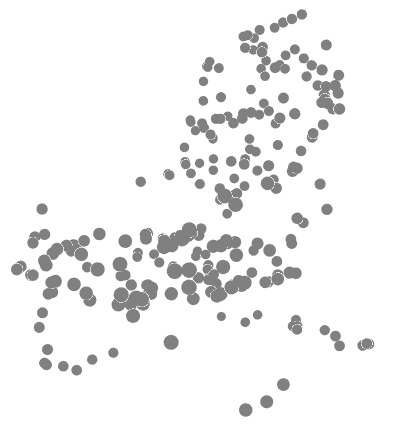}\\
  \caption{Low-Res Input}
  \end{subfigure}
\end{minipage}
\begin{minipage}[c]{0.42\linewidth}
  \begin{subfigure}[b]{0.99\linewidth}
  \centering
   \includegraphics[height=0.55\linewidth, width=0.45\linewidth]{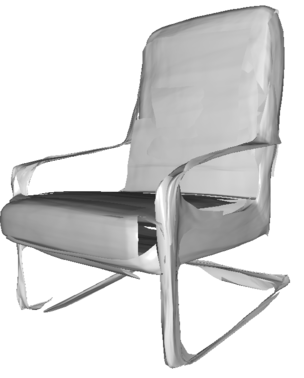}\\
   \includegraphics[height=0.55\linewidth, width=0.45\linewidth]{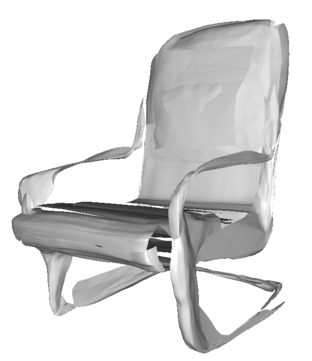}\\
  \caption{High-Res reconstruction}
  \end{subfigure}
\end{minipage}

\caption{
\textbf{Super resolution.} Our approach can generate meshes at arbitrary resolutions, and the pointnet encoder \cite{qi2016pointnet} can take pointclouds of varying resolution as input. Given the same shape sampled at the training resolution of 2500, or 10 times less points, we generate high resolution meshes with 122500 vertices. This can be viewed as the 3D equivalent of super-resolution on 2D pixels.}
\label{fig:super_resolution}
\end{figure*}

\paragraph{Details on the comparison against HSP \cite{Hane:2017}} 

We perform a quantitative comparison against an octree-based state of the art method. AtlasNet is trained with 25 learned parameterizations on the same data as their publicly available trained model\footnote{\url{https://github.com/chaene/hsp}.}. 100 random samples are drawn from each category from the test split. We evaluated the the quality of the reconstruction using the Chamfer distance on the unnormalized meshes, and the metro distance. Voxelised versions of meshes often appear inflated. This bias can appear for HSP, where we observed that the generated meshes were slightly larger than the original meshes. We ran an ICP alignment procedure on the generated meshes for both methods to remove this bias. In table \ref{tab:HSP}, we report per category results. As AtlasNet was specifically trained to optimise the chamfer distance, we outperform HSP in every category. AtlasNet also outperforms HSP in metro distance in each category for the metro distance, for which none of the two algorithm where trained to optimise. List of sampled used, ans trained model for AtlasNet are available in the github repository.

\paragraph{Limitations and future work}

Our results have limitations that lead to many open question and perspective for future work.
First, the patches for our generated shapes are not guaranteed to be connected (except if the surface the input points are sampled from is already closed, as in the sphere experiment). An open question is how to effectively stitch the patches together to form a closed shape.
Second, we have demonstrated results on synthetic object shapes. Ideally, we would like to extend to entire real scenes.
Third, we have optimized the parameterization of the generated meshes post-hoc. It would be good to directly learn to generate the surfaces with low distortion parameterizations.
Fourth, this work generates surfaces by minimizing an energy computed from point clouds. An open question is how to define a loss on meshes that is easy to optimize?
Finally, as the atlases provide promising correspondences across different shapes, an interesting future direction is to leverage them for shape recognition and segmentation.
\vspace*{-2pt}

\begin{table*}[hb!]
\centering
{
\small
  \begin{tabular}{l|c|c|c|c|c|c|c|c|c|c|c|c|c|c}
   &  pla. &  ben. &  cab. &  car &  cha. &  mon. &  lam. &  spe. &  fir. &  cou. &  tab. &  cel. &  wat. &  mean \\
  \hline
 {Baseline PSR} & 2.71 & 2.12 & 1.98 & 2.24 & 2.68 & 1.78 & 2.58 & 2.29 & 1.03 & 1.90 & 2.66 & 1.15 & 2.46 & 2.12 \\
 {Baseline PSR PA} & 1.38 & 1.97 & 1.75 & 2.04 & 2.08 & 1.53 & 2.51 & 2.25 & 1.46 & 1.57 & 2.06 & 1.15 & 1.80 & 1.82 \\
 Ours 1 patch & 1.11 & 1.41 & 1.70 & 1.93 & 1.76 & 1.35 & 2.01 & 2.30 & 1.01 & 1.46 & 1.46 & 0.87 & 1.46 & 1.53  \\
 Ours 1 sphere & 1.03 & 1.33 & 1.64 & 1.99 & 1.76 & 1.30 & 2.06 & 2.33 & 0.93 & 1.41 & 1.59 & 0.79 & 1.54 & 1.52  \\
 Ours 5 patch & 0.99 & 1.36 & 1.65 & \textbf{1.90} & 1.79 & 1.28 & 2.00 & 2.27 & 0.92 & \textbf{1.37} & 1.57 & \textbf{0.76} & 1.40 & 1.48 \\
 Ours 25 patch & \textbf{0.96} & 1.35 & 1.63 & 1.96 & 1.49 & \textbf{1.22} & \textbf{1.86} & \textbf{2.22} & 0.93 & 1.36 & \textbf{1.31} & 1.41 & \textbf{1.35} & 1.47 \\
 Ours 125 patch & 1.01 & \textbf{1.30} & \textbf{1.58} & \textbf{1.90} & \textbf{1.36} & 1.29 & 1.95 & 2.29 & \textbf{0.85} & 1.38 & 1.34 & \textbf{0.76} & 1.37 & \textbf{1.41} \\

\hline
  \end{tabular}

  }
  \caption{\textbf{Auto-Encoder (per category).} The mean is taken category-wise. The Metro Distance is reported, multiplied by $10$. The meshes were contructed by propagating the patch grid edges.
  }
  \label{tab:SVR_AE}
\end{table*}

\vspace*{-8pt}
\begin{table*}
\centering
{
\small
  \begin{tabular}{l|c|c|c|c|c|c|c|c|c|c|c|c|c|c}
   &  pla. &  ben. &  cab. &  car &  cha. &  mon. &  lam. &  spe. &  fir. &  cou. &  tab. &  cel. &  wat. &  mean \\
  \hline
 {Baseline} &  1.11 & 1.46 & 1.91 & 1.59 & 1.90 & 2.20 & 3.59 & 3.07 & 0.94 & 1.83 & 1.83 & 1.71 & 1.69 & 1.91  \\
 {Baseline +  normal} &  1.25 & 1.73 & 2.19 & 1.74 & 2.19 & 2.52 & 3.89 & 3.51 & 0.98 & 2.13 & 2.17 & 1.87 & 1.88 & 2.15  \\
 {Ours 1 patch} &  1.04 &  1.43 &  1.79 &  2.28 &  1.97 &  1.83 &  3.06 &  2.95 &  0.76 &  1.90 &  1.95 &  1.29 &  1.69 &  1.84 \\
 Ours 1 sphere & 0.98 &  1.31 &  2.02 &  1.75 &  1.81 &  1.83 &  2.59 &  2.94 &  0.69 &  1.73 &  1.88 &  1.30 &  1.51 & 1.72 \\
 {Ours 5 patch} &  0.96 &  1.21 &  \textbf{1.64} &  1.76 &  1.60 &  \textbf{1.66} &  2.51 &  \textbf{2.55} &  0.68 &  1.64 &  1.52 &  \textbf{1.25} &  1.46 &  1.57 \\
 {Ours 25 patch} &  0.87 &  1.25 &  1.78 &  1.58 &  1.56 &  1.72 &  2.30 &  2.61 &  0.68 &  1.83 &  1.52 &  1.27 &  1.33 &  1.56 \\
 {Ours 125 patch} &  \textbf{0.86} &  \textbf{1.15} &  1.76 &  \textbf{1.56} &  \textbf{1.55} &  1.69 &  \textbf{2.26} &  \textbf{2.55} &  \textbf{0.59} &  \textbf{1.69} &  \textbf{1.47} &  1.31 &  \textbf{1.23} & \textbf{1.51}\\

\hline
  \end{tabular}

  }
  \caption{\textbf{Auto-Encoder (per category).} The mean is taken category-wise. The Chamfer Distance is reported, multiplied by $10^3$.
  }
  \label{tab:SVR_AE_chamfer}
\end{table*}

\begin{table*}[hb!]
\centering
{
\small
  \begin{tabular}{c|l|c|c|c|c|c|c|c|c|c|c|c|c|c|c}
  & &  pla. &  ben. &  cab. &  car &  cha. &  mon. &  lam. &  spe. &  fir. &  cou. &  tab. &  cel. &  wat. &  mean \\
  \hline
 metro & {HSP} &1.10 & 1.84 & 1.28 & 1.06 & 1.61 & 1.66 & 1.93 & 1.77 & 1.05 & 1.37 & 1.93 & 1.39 & 1.34 & 1.49 \\
  & Ours 25 patch & \textbf{0.77} & \textbf{1.01} & \textbf{1.04} & \textbf{0.92} & \textbf{1.19} & \textbf{1.22} & \textbf{1.26} & \textbf{1.46} & \textbf{0.95} & \textbf{1.19} & \textbf{1.27} & \textbf{0.83} & \textbf{1.09} & \textbf{1.09}\\
 \hline
 chamfer & {HSP} & 2.60 & 17.4 & 14.3 & 1.77 & 10.0 & 19.4 & 9.46 & 21.7 & 2.34 & 12.9 & 20.2 & 13.2 & 4.89 &11.6 \\
  & Ours 25 patch & \textbf{1.33} & \textbf{14.1} & \textbf{12.5} & \textbf{1.29} & \textbf{7.23} & \textbf{17.5} & \textbf{6.99} & \textbf{17.8} & \textbf{1.69} & \textbf{11.2} & \textbf{17.0} & \textbf{10.6} & \textbf{4.20} & \textbf{9.52}\\
\hline
  \end{tabular}

  }
  \caption{{\bf Single-view reconstruction.} Quantitative comparison against HSP \cite{Hane:2017}, a state of the art octree-based method. The average error is reported, on 100 shapes from each category. The Chamfer Distance reported is computed on $10^4$ points, and multiplied by $10^3$. The Metro distance is multiplied by 10.}
  \label{tab:HSP}
\end{table*}

\begin{table*}
\centering
{
\small
  \begin{tabular}{l|c}
  Weight Decay &  Ours : 25 patches \\
  \hline
 {$10^{-3}$} &  8.57   \\
 {$10^{-4}$} &  4.84   \\
 {$10^{-5}$} &  3.42   \\
 {$0$} &  1.56   \\

\hline
  \end{tabular}

  }
  \caption{\textbf{Regularization on Auto-Encoder (per category).} The mean is taken category-wise. The Chamfer Distance is reported, multiplied by $10^3$.
  }
  \label{tab:SVR_regul}
\end{table*}

\newpage
\begin{figure*}[t!]
\centering
\begin{subfigure}[b]{0.24\linewidth}
\centering
 \includegraphics[height=0.40\linewidth]{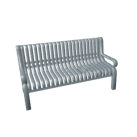}\\
 \includegraphics[height=0.40\linewidth]{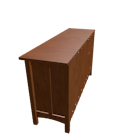}\\
 \includegraphics[height=0.40\linewidth]{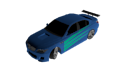}\\
 \includegraphics[height=0.40\linewidth]{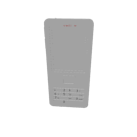}\\
 \includegraphics[height=0.40\linewidth]{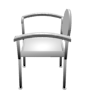}\\
 \includegraphics[height=0.40\linewidth]{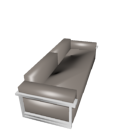}\\
 \includegraphics[height=0.40\linewidth]{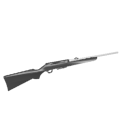}\\
 \includegraphics[height=0.40\linewidth]{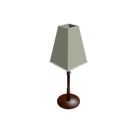}\\
 \includegraphics[height=0.40\linewidth]{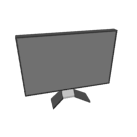}\\
 \includegraphics[height=0.40\linewidth]{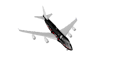}\\
 \includegraphics[height=0.40\linewidth]{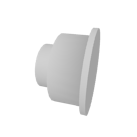}\\
 \includegraphics[height=0.40\linewidth]{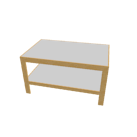}\\
 \includegraphics[height=0.40\linewidth]{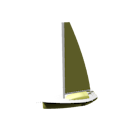}
\caption{input}
\end{subfigure}
\begin{subfigure}[b]{0.24\linewidth}
\centering
 \includegraphics[height=0.30\linewidth]{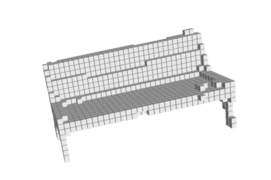}\\[0.1\linewidth]
 \includegraphics[height=0.30\linewidth]{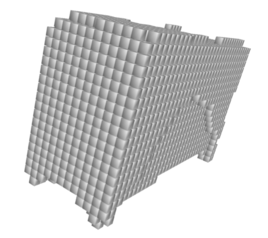}\\[0.1\linewidth]
 \includegraphics[height=0.30\linewidth]{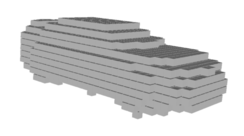}\\[0.1\linewidth]
 \includegraphics[height=0.30\linewidth]{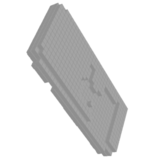}\\[0.1\linewidth]
 \includegraphics[height=0.30\linewidth]{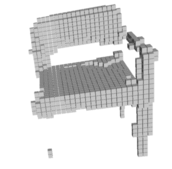}\\[0.1\linewidth]
 \includegraphics[height=0.30\linewidth]{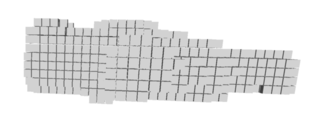}\\[0.1\linewidth]
 \includegraphics[height=0.30\linewidth]{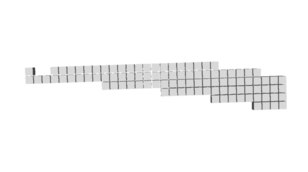}\\[0.1\linewidth]
 \includegraphics[height=0.30\linewidth]{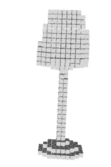}\\[0.1\linewidth]
 \includegraphics[height=0.30\linewidth]{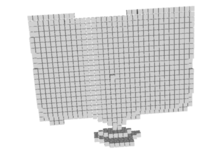}\\[0.1\linewidth]
 \includegraphics[height=0.30\linewidth]{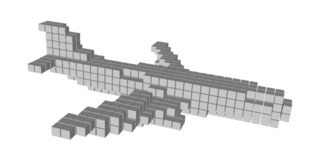}\\[0.1\linewidth]
 \includegraphics[height=0.30\linewidth]{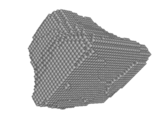}\\[0.1\linewidth]
 \includegraphics[height=0.30\linewidth]{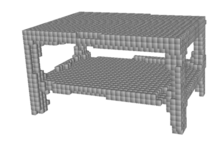}\\[0.1\linewidth]
 \includegraphics[height=0.30\linewidth]{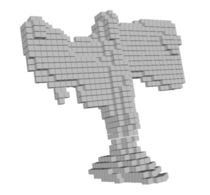}\\[0.1\linewidth]
\caption{3D-R2N2}
\end{subfigure}
\begin{subfigure}[b]{0.24\linewidth}
\centering
 \includegraphics[height=0.30\linewidth]{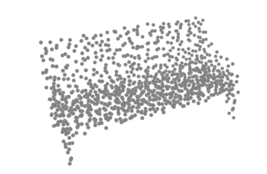}\\[0.1\linewidth]
 \includegraphics[height=0.30\linewidth]{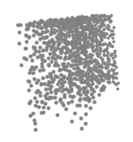}\\[0.1\linewidth]
 \includegraphics[height=0.30\linewidth]{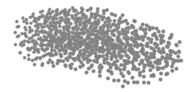}\\[0.1\linewidth]
 \includegraphics[height=0.30\linewidth]{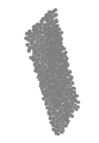}\\[0.1\linewidth]
 \includegraphics[height=0.30\linewidth]{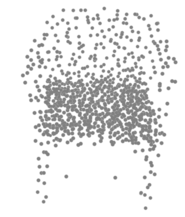}\\[0.1\linewidth]
 \includegraphics[height=0.30\linewidth]{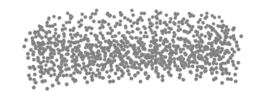}\\[0.1\linewidth]
 \includegraphics[height=0.30\linewidth]{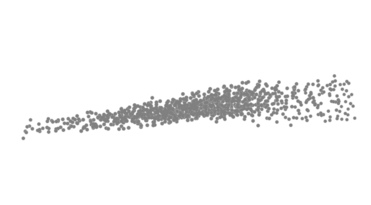}\\[0.1\linewidth]
 \includegraphics[height=0.30\linewidth]{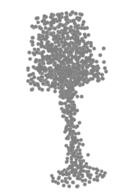}\\[0.1\linewidth]
 \includegraphics[height=0.30\linewidth]{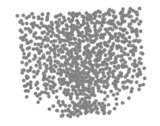}\\[0.1\linewidth]
 \includegraphics[height=0.30\linewidth]{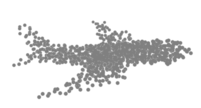}\\[0.1\linewidth]
 \includegraphics[height=0.30\linewidth]{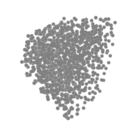}\\[0.1\linewidth]
 \includegraphics[height=0.30\linewidth]{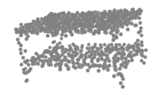}\\[0.1\linewidth]
 \includegraphics[height=0.30\linewidth]{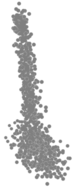}\\[0.1\linewidth]
\caption{PSG}
\end{subfigure}
\begin{subfigure}[b]{0.24\linewidth}
\centering
 \includegraphics[height=0.30\linewidth]{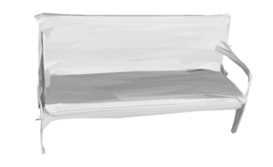}\\[0.1\linewidth]
 \includegraphics[height=0.30\linewidth]{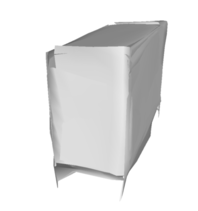}\\[0.1\linewidth]
 \includegraphics[height=0.30\linewidth]{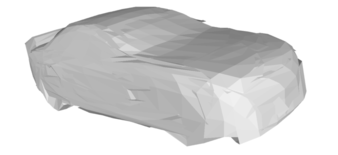}\\[0.1\linewidth]
 \includegraphics[height=0.30\linewidth]{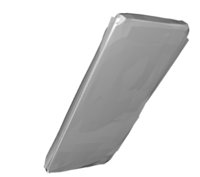}\\[0.1\linewidth]
 \includegraphics[height=0.30\linewidth]{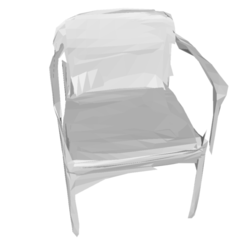}\\[0.1\linewidth]
 \includegraphics[height=0.30\linewidth]{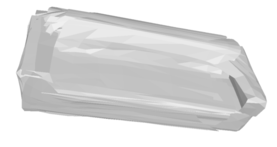}\\[0.1\linewidth]
 \includegraphics[height=0.30\linewidth]{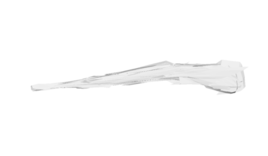}\\[0.05\linewidth]
 \includegraphics[height=0.30\linewidth]{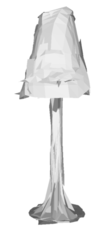}\\[0.1\linewidth]
 \includegraphics[height=0.30\linewidth]{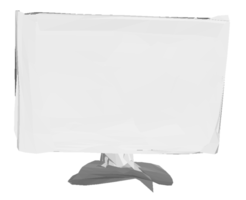}\\[0.1\linewidth]
 \includegraphics[height=0.30\linewidth]{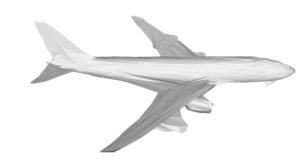}\\[0.1\linewidth]
 \includegraphics[height=0.30\linewidth]{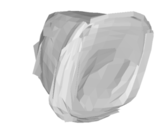}\\[0.1\linewidth]
 \includegraphics[height=0.30\linewidth]{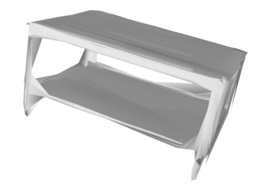}\\[0.1\linewidth]
 \includegraphics[height=0.30\linewidth]{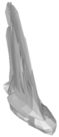}\\[0.1\linewidth]
\caption{Ours}
\end{subfigure}
\caption{
{\bf Single-view reconstruction comparison:}
From a 2D RGB image (a), 3D-R2N2 reconstructs a voxel-based 3D model (b), PointSetGen a point cloud based 3D model (c), and our \ournet{} a triangular mesh (d).}
\vspace{-10ptpt}
  \label{fig:svr_comparison_2_2}
\end{figure*}

\newpage
\begin{figure*}[t!]
\centering
\begin{subfigure}[b]{0.24\linewidth}
\centering
 \includegraphics[height=0.30\linewidth]{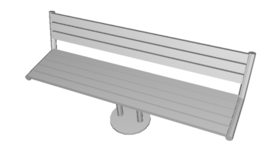}\\[0.1\linewidth]
 \includegraphics[height=0.30\linewidth]{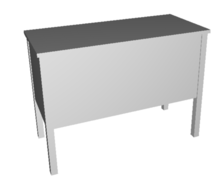}\\[0.1\linewidth]
 \includegraphics[height=0.30\linewidth]{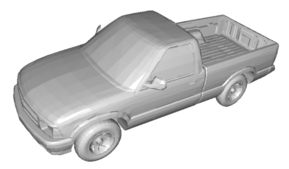}\\[0.1\linewidth]
 \includegraphics[height=0.30\linewidth]{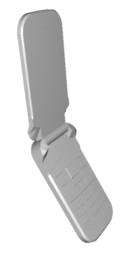}\\[0.1\linewidth]
 \includegraphics[height=0.30\linewidth]{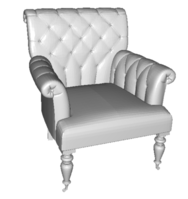}\\[0.1\linewidth]
 \includegraphics[height=0.30\linewidth]{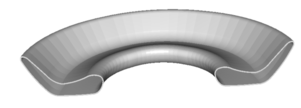}\\[0.1\linewidth]
 \includegraphics[height=0.30\linewidth]{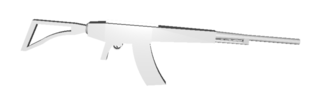}\\[0.1\linewidth]
 \includegraphics[height=0.30\linewidth]{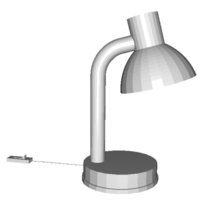}\\[0.1\linewidth]
 \includegraphics[height=0.30\linewidth]{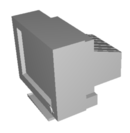}\\[0.1\linewidth]
 \includegraphics[height=0.30\linewidth]{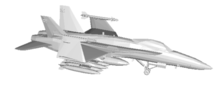}\\[0.1\linewidth]
 \includegraphics[height=0.30\linewidth]{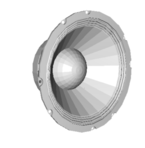}\\[0.1\linewidth]
 \includegraphics[height=0.30\linewidth]{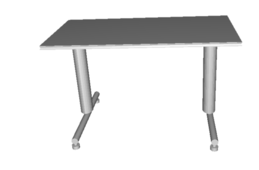}\\[0.1\linewidth]
 \includegraphics[height=0.30\linewidth]{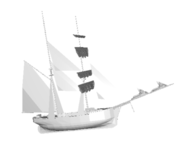}\\[0.1\linewidth]
\caption{Ground truth}
\end{subfigure}
\begin{subfigure}[b]{0.24\linewidth}
\centering
 \includegraphics[height=0.30\linewidth]{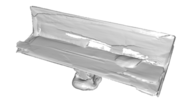}\\[0.1\linewidth]
 \includegraphics[height=0.30\linewidth]{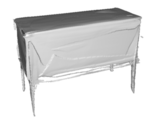}\\[0.1\linewidth]
 \includegraphics[height=0.30\linewidth]{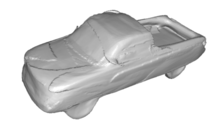}\\[0.1\linewidth]
 \includegraphics[height=0.30\linewidth]{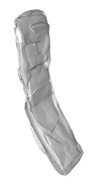}\\[0.1\linewidth]
 \includegraphics[height=0.30\linewidth]{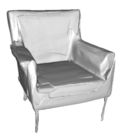}\\[0.1\linewidth]
 \includegraphics[height=0.30\linewidth]{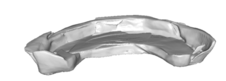}\\[0.1\linewidth]
 \includegraphics[height=0.30\linewidth]{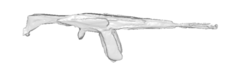}\\[0.1\linewidth]
 \includegraphics[height=0.30\linewidth]{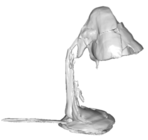}\\[0.1\linewidth]
 \includegraphics[height=0.30\linewidth]{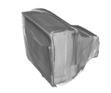}\\[0.1\linewidth]
 \includegraphics[height=0.30\linewidth]{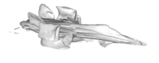}\\[0.1\linewidth]
 \includegraphics[height=0.30\linewidth]{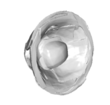}\\[0.1\linewidth]
 \includegraphics[height=0.30\linewidth]{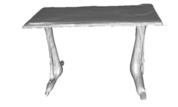}\\[0.1\linewidth]
 \includegraphics[height=0.30\linewidth]{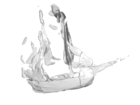}\\[0.1\linewidth]
\caption{PSR on ours}
\end{subfigure}
\begin{subfigure}[b]{0.24\linewidth}
\centering
 \includegraphics[height=0.30\linewidth]{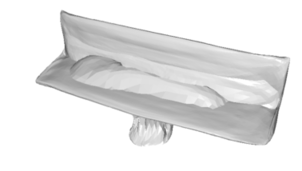}\\[0.1\linewidth]
 \includegraphics[height=0.30\linewidth]{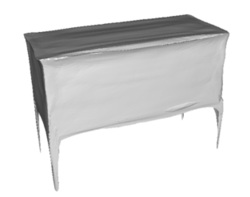}\\[0.1\linewidth]
 \includegraphics[height=0.30\linewidth]{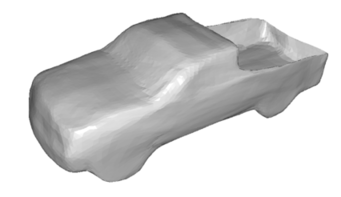}\\[0.1\linewidth]
 \includegraphics[height=0.30\linewidth]{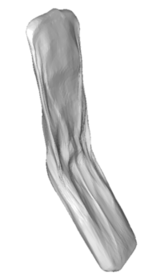}\\[0.1\linewidth]
 \includegraphics[height=0.30\linewidth]{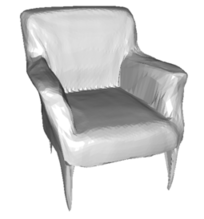}\\[0.1\linewidth]
 \includegraphics[height=0.30\linewidth]{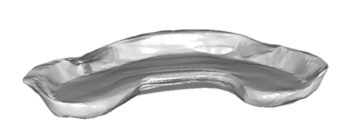}\\[0.1\linewidth]
 \includegraphics[height=0.30\linewidth]{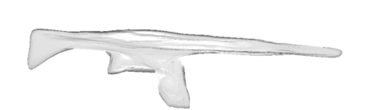}\\[0.1\linewidth]
 \includegraphics[height=0.30\linewidth]{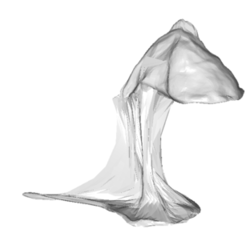}\\[0.1\linewidth]
 \includegraphics[height=0.30\linewidth]{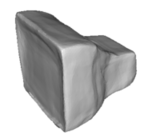}\\[0.1\linewidth]
 \includegraphics[height=0.30\linewidth]{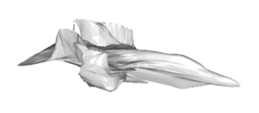}\\[0.1\linewidth]
 \includegraphics[height=0.30\linewidth]{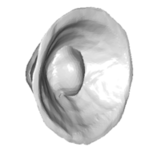}\\[0.1\linewidth]
 \includegraphics[height=0.30\linewidth]{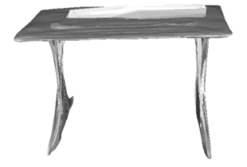}\\[0.1\linewidth]
 \includegraphics[height=0.30\linewidth]{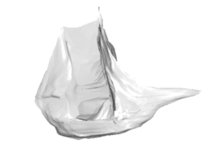}\\[0.1\linewidth]
\caption{Ours sphere}
\end{subfigure}
\begin{subfigure}[b]{0.24\linewidth}
\centering
 \includegraphics[height=0.30\linewidth]{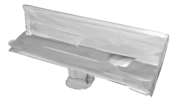}\\[0.1\linewidth]
 \includegraphics[height=0.30\linewidth]{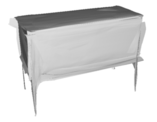}\\[0.1\linewidth]
 \includegraphics[height=0.30\linewidth]{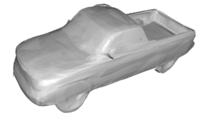}\\[0.1\linewidth]
 \includegraphics[height=0.30\linewidth]{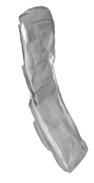}\\[0.1\linewidth]
 \includegraphics[height=0.30\linewidth]{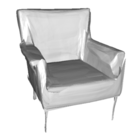}\\[0.1\linewidth]
 \includegraphics[height=0.30\linewidth]{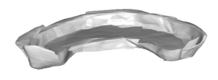}\\[0.1\linewidth]
 \includegraphics[height=0.30\linewidth]{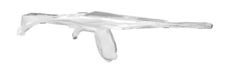}\\[0.1\linewidth]
 \includegraphics[height=0.30\linewidth]{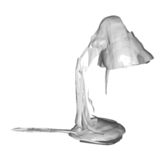}\\[0.1\linewidth]
 \includegraphics[height=0.30\linewidth]{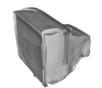}\\[0.1\linewidth]
 \includegraphics[height=0.30\linewidth]{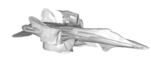}\\[0.1\linewidth]
 \includegraphics[height=0.30\linewidth]{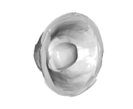}\\[0.1\linewidth]
 \includegraphics[height=0.30\linewidth]{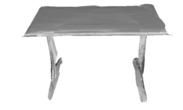}\\[0.1\linewidth]
 \includegraphics[height=0.30\linewidth]{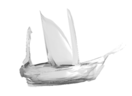}\\[0.1\linewidth]
\caption{Ours 25}
\end{subfigure}
\caption{
{\bf Autoencoder comparison:}
We compare the original meshes (a) to meshes obtained by running PSR (b) on the dense point cloud sampled from our generated mesh,  and to our method generating a surface from a sphere (c), and 25 (d) learnable parameterizations.} %
\vspace{-10ptpt}
  \label{fig:svr_comparison_2}
\end{figure*}

\newpage
\begin{figure*}[t]
 \includegraphics[width=0.9\linewidth]{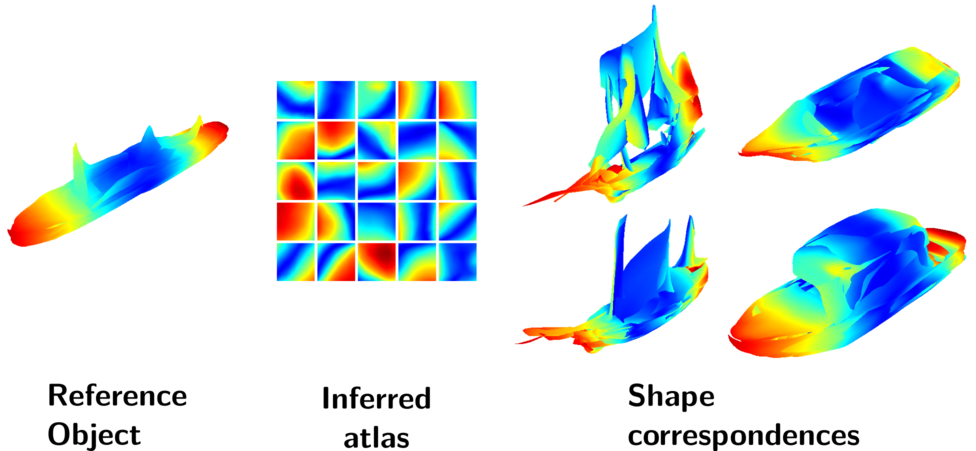}
\caption{{\bf Shape correspondences:} a reference watercraft (left) is colored by distance to the center, with the jet colormap. We transfer the surface colors to the inferred atlas for the reference shape (middle). Finally, we transfer the atlas colors to other shapes (right). Notice that we get semantically meaningful correspondences, without any supervision from the dataset on semantic information. All objects are generated by the autoencoder, with 25 learned parametrizations.\label{fig:corresp_1}}
\end{figure*}
\begin{figure*}[b]
 \includegraphics[width=0.9\linewidth]{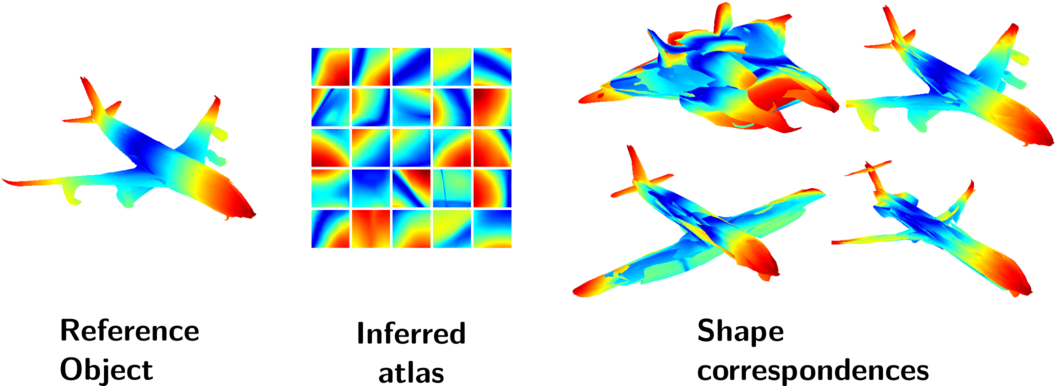}
\caption{{\bf Shape correspondences:} a reference plane (left) is colored by distance to the center, with the jet colormap. We transfer the surface colors to the inferred atlas for the reference shape (middle). Finally, we transfer the atlas colors to other shapes (right). Notice that we get semantically meaningful correspondences, such as the nose and tail of the plane, and the tip of the wings, without any supervision from the dataset on semantic information. All objects are generated by the autoencoder, with 25 learned parametrizations.\label{fig:corresp_2}}
\end{figure*}

\newpage
\begin{figure*}[t!]
\centering
\begin{subfigure}[b]{0.98\linewidth}
\centering
 \includegraphics[width=0.3\linewidth]{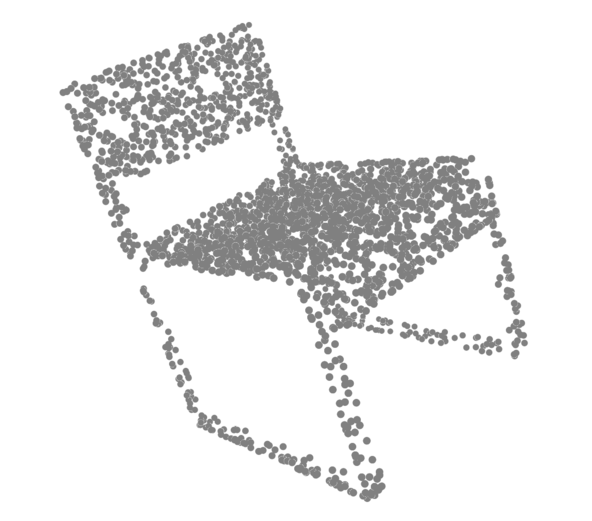}
 \includegraphics[width=0.3\linewidth]{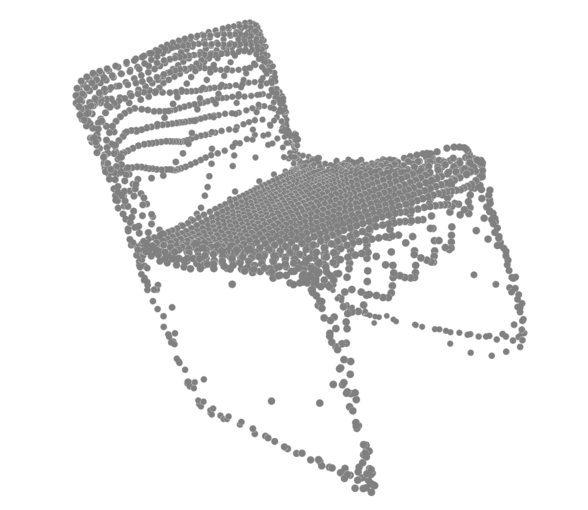}
 \includegraphics[width=0.3\linewidth]{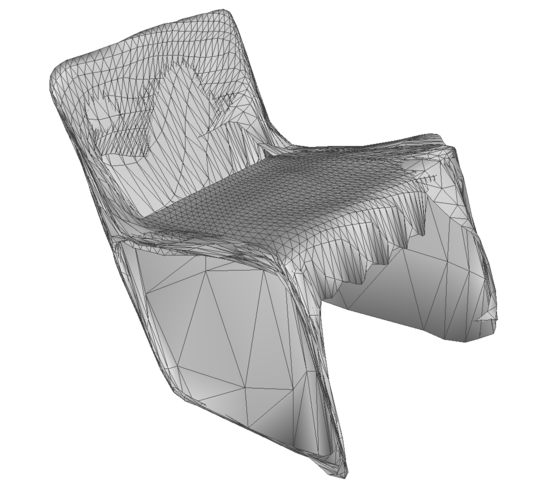}
\caption{ {\bf Excess of distortion. } Notice how, compared to the original point cloud (left), the generated pointcloud (middle) with 1 learned parameterization is valid, but the mapping from squares to surfaces enforces too much distortion leading to error when propagating the grid edges in 3D (right).}
\end{subfigure}
\\
\begin{subfigure}[b]{0.98\linewidth}
\centering
 \includegraphics[width=0.3\linewidth]{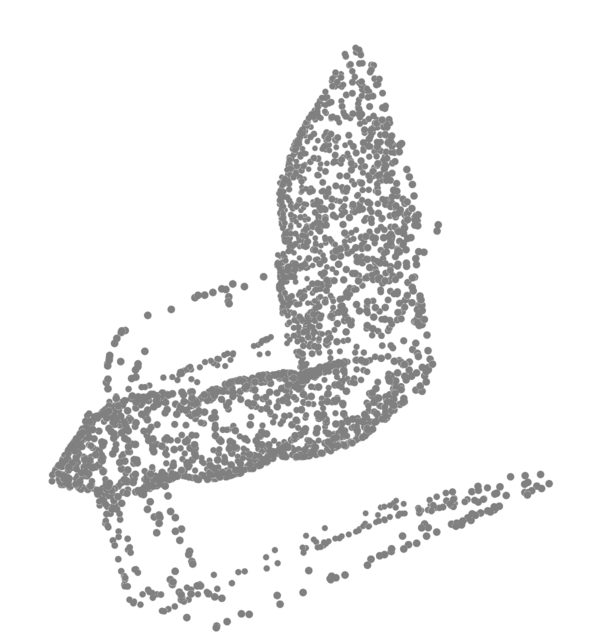}
 \includegraphics[width=0.3\linewidth]{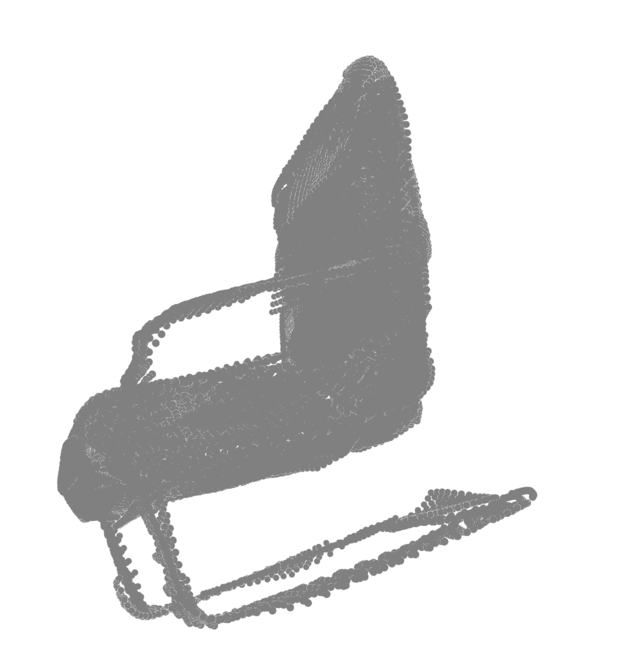}
 \includegraphics[width=0.3\linewidth]{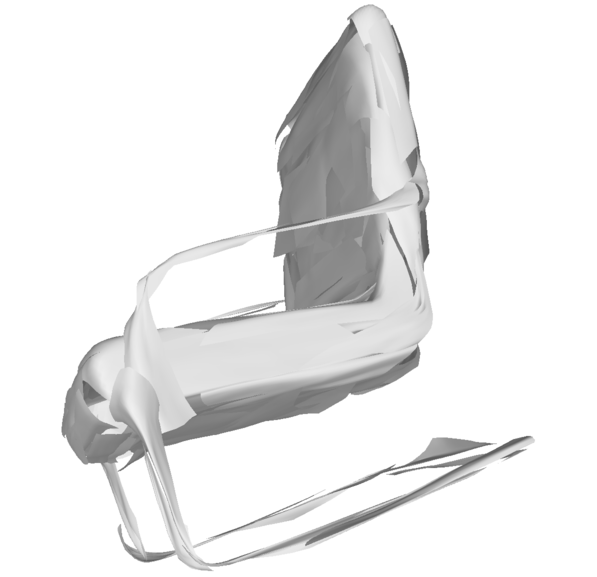}
\caption{{\bf Topological issues. } Notice how, compared to the original point cloud (left), the generated pointcloud (middle) with 125 learned parameterizations is valid, but the 125 generated surfaces overlap and are not stiched together (right).}
\end{subfigure}
\caption{
{\bf Limitations.} Two main artifacts are highlighted : (a) Excess of distortion when too small a number of learned parameterizations is used, and (b) growing errors in the topology of the reconstructed mesh as the number of learned parameterization increases.}
  \label{fig:limitation}
\end{figure*}

\begin{figure*}[t!]

\centering

\begin{subfigure}[b]{\linewidth}
 ~~~~~~\includegraphics[height=0.35\linewidth, width=0.19\linewidth]{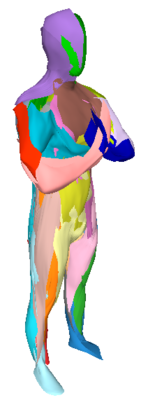}~~~~~~~~~~
  \includegraphics[height=0.35\linewidth, width=0.19\linewidth]{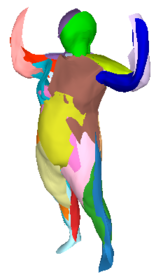}~~~~~~~~~~
 \includegraphics[height=0.35\linewidth, width=0.19\linewidth]{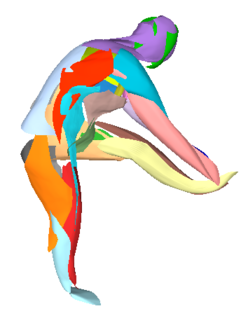}~~
\end{subfigure}

\caption{
{\bf Deformable shapes.} Our method learned on 250 shapes from the FAUST dataset to reconstructs a human in different poses. Each color represent one of the 25 parametrizations.}
  \label{fig:faust}
\end{figure*} 

\end{document}